\documentclass[11pt]{article}
\usepackage[T1]{fontenc}
\usepackage[utf8]{inputenc}
\usepackage{setspace}
\usepackage{newtxtext,newtxmath} 
\usepackage[margin=1.5in]{geometry}
\usepackage{url}
\usepackage{array}
\usepackage{graphicx}
\usepackage{subcaption}
\usepackage{color}
\makeatletter
\@ifundefined{openbox}{}{}
\makeatother

\usepackage{amsmath,amssymb,amsthm,mathtools}
\usepackage{bm}
\usepackage{booktabs}
\usepackage{tabularx}
\usepackage{enumitem}
\usepackage[dvipsnames]{xcolor}

\usepackage{multirow}
\usepackage{authblk}
\usepackage[round,compress]{natbib}
\usepackage{tikz}
\usepackage{pgfplots}
\pgfplotsset{compat=1.18}
\usepackage{float}     
\usepackage{placeins}  

\newtheorem{theorem}{Theorem}[section]
\newtheorem{lemma}[theorem]{Lemma}
\newtheorem{proposition}[theorem]{Proposition}
\newtheorem{corollary}[theorem]{Corollary}
\theoremstyle{definition}
\newtheorem{definition}[theorem]{Definition}
\newtheorem{assumption}[theorem]{Assumption}

\theoremstyle{remark}
\newtheorem{remark}[theorem]{Remark}

\newcommand{\E}{\mathbb{E}}
\newcommand{\Var}{\mathrm{Var}}
\newcommand{\Cov}{\mathrm{Cov}}

\newcommand{\Tr}{\mathrm{Tr}}

\newcommand{\R}{\mathbb{R}}
\newcommand{\N}{\mathbb{N}}

\newcommand{\cF}{\mathcal{F}}
\newcommand{\cU}{\mathcal{U}}

\newcommand{\cO}{\mathcal{O}}

\newcommand{\cN}{\mathcal{N}}

\newcommand{\cZ}{\mathcal{Z}}
\newcommand{\cA}{\mathcal{A}}

\newcommand{\Gstar}{G^*}

\newcommand{\psiFn}{\psi}

\DeclareMathOperator{\op}{op}
\DeclareMathOperator{\eff}{eff}
\DeclareMathOperator{\loc}{loc}

\newcolumntype{Y}{>{\raggedright\arraybackslash}X}

\let\oldproof\proof
\let\oldendproof\endproof
\renewenvironment{proof}[1]{\oldproof}{\oldendproof}
\newcommand{\Halmos}{\qedhere}
\newcommand{\up}{}
\newcommand{\down}{}
\newcommand{\FIGURE}[3]{%
  \centering #1\caption{#2}\par\medskip{\footnotesize #3}}
\newcommand{\TABLE}[3]{%
  \centering\caption{#1}\medskip #2\par\medskip{\footnotesize #3}}

\usepackage{mathrsfs}

\usepackage[colorlinks=true,allcolors=green!60!black]{hyperref}

\title{\textbf{Mini-Batch Covariance, Diffusion Limits, and Oracle Complexity\\
in Stochastic Gradient Descent:\\
A Sampling-Design Perspective}}

\author[1]{Daniel Zantedeschi}
\author[2]{Kumar Muthuraman}
\affil[1]{Muma College of Business, University of South Florida, Tampa, FL, USA\\
\texttt{danielz@usf.edu}}
\affil[2]{McCombs School of Business, University of Texas at Austin, Austin, TX, USA\\
\texttt{kumar.muthuraman@mccombs.utexas.edu}}
\date{\today}

\begin{document}

\maketitle

\begin{abstract}
Stochastic gradient descent (SGD) is central to simulation optimization,
stochastic programming, and online M-estimation, where sampling effort is a
decision variable.  We study the mini-batch gradient noise as a
sampling-design object.  Under exchangeable fresh-sampling mini-batches, the
conditional covariance given the de~Finetti directing measure $\mu$ is
$b^{-1}G_\mu(\theta)$, and under identifiability the projected population
object is $b^{-1}G^\star(\theta)$---projected Fisher information for correctly
specified likelihoods, the sandwich partner of the Hessian otherwise.  This
identification fixes the noise matrix entering the diffusion analysis of
constant-step SGD: the raw iterate path has a deterministic fluid limit, and
the $\sqrt{b/\eta}$-scaled fluctuations satisfy a functional CLT with noise
covariance $G^\star$; near a nondegenerate optimum the limit is
Ornstein--Uhlenbeck, and its Lyapunov covariance scaled by $\eta/b$ matches
the linearized discrete recursion at leading order.  Under a curvature--noise
compatibility condition $\mu_F>0$, we prove $1/N$ mean-square upper bounds
and an i.i.d.\ parametric Fisher van~Trees lower bound of the same rate
order, with oracle-complexity guarantees depending on an effective dimension
$d_{\mathrm{eff}}$ and condition number $\kappa_F$.  Numerical experiments
verify the identification and confirm the Lyapunov predictions in direct
SGD.
\end{abstract}


\section{Introduction}
\label{sec:intro}

Stochastic gradient descent (SGD) is the computational workhorse behind a large
class of operations research and management science problems, including simulation
optimization, large-scale stochastic programming, and data-driven estimation in
service and supply-chain systems. In these settings, \emph{sampling effort is a
decision variable}: each iterate trades off additional replications (or
scenarios) against the number of updates that can be executed under a fixed
budget. Several widely observed phenomena depend on this trade-off: small
mini-batches often dominate large batches in wall-clock efficiency; constant-step
SGD exhibits steady-state ``error floors'' and implicit regularization; and the
linearized local equilibrium of plain SGD exhibits directional structure
related to the local curvature--noise geometry, despite the drift being
unpreconditioned.

Classical stochastic approximation and diffusion-approximation
analyses---small-step SDE surrogates, Ornstein--Uhlenbeck linearization,
Lyapunov stationary covariance, $\sqrt{t}$-CLT scaling---are
well-developed OR tools for analyzing such recursions once the
gradient-noise covariance is specified
\citep{kushner2003stochastic,harrison1985brownian,whitt2002stochastic}.
We do not reprove that machinery.  Our contribution is a sampling-design
reinterpretation of the noise input: under exchangeable fresh-sampling
mini-batches, the conditional mini-batch covariance is determined by
the sampling mechanism at $b^{-1}$ times the per-sample gradient
covariance, a statement that is elementary but carries structural
consequences for operational decisions about sampling budgets.  Writing the
identified matrix as $G^\star(\theta)$---projected Fisher information
at the correctly specified point, and more generally the sandwich
partner of the Hessian $H^\star$---lets us analyze constant-step SGD in
a fixed statistical metric rather than a Euclidean one, with the
diffusion coefficient determined by the sampling design rather than
postulated.  Under a curvature--noise
compatibility condition $\mu_F>0$
(Assumption~\ref{ass:muF_pos}; not implied by $H^\star,F^\star\succ 0$
alone), the resulting local rates in the Fisher case match an i.i.d.\
van~Trees lower bound in rate order, and batch size becomes a
statistical design variable subject to a sampling budget.

\subsection{Motivation: Simulation Optimization as Sampling Control}

Consider the canonical simulation-optimization objective
\begin{equation}
\min_{\theta\in\Theta}\; L(\theta)\;:=\;\E\big[h(\theta,\xi)\big],
\label{eq:simopt-intro}
\end{equation}
where $\xi$ denotes random inputs (arrivals, service times, demand shocks) and
$h(\theta,\xi)$ is the induced cost. A standard stochastic-gradient estimator
uses $b$ simulation replications,
\begin{equation}
g_B(\theta) \;=\; \frac{1}{b}\sum_{i=1}^b \nabla_\theta h(\theta,\xi_i),
\qquad \xi_i \stackrel{\text{i.i.d.}}{\sim} P,
\label{eq:mini-batch-simopt}
\end{equation}
followed by $\theta^{+}=\theta-\eta\,g_B(\theta)$. Classical variance-reduction
intuition focuses on the scalar variance of a component or of a scalarized
performance measure, suggesting that larger $b$ should always help.
In practice, under fixed sampling budgets, small $b$ often wins.

The key reason is geometric: the noise in $g_B(\theta)$ is not scalar but a
\emph{matrix}-valued covariance that encodes which directions are
intrinsically informative and which are weakly identified by the underlying
experiment.  The same sampling-control perspective applies to three canonical
MOR problem classes:
(i)~\emph{simulation optimization} (queuing, inventory, revenue management),
where each mini-batch gradient requires $b$ simulation replications and the
replication budget is the binding constraint;
(ii)~\emph{stochastic programming and SAA}, where scenario-based gradient
estimates are drawn from an empirical distribution and the number of
scenarios per iterate is a design variable subject to a total scenario
budget; and
(iii)~\emph{online M-estimation in service systems}, where per-transaction
loss gradients arrive in a stream and the batch window trades off estimation
quality against decision latency.  In each case, sampling effort is
explicit, $b$ is typically small relative to $d$, and the question ``how
many samples per iterate?'' is operationally binding.

\subsection{Contributions and scope}
\label{subsec:core_insight}

Under exchangeable fresh-sampling mini-batches, the conditional
mini-batch covariance given the de~Finetti directing measure $\mu$ is
$\Cov(g_B(\theta)\mid\mu)=b^{-1}G_\mu(\theta)$
(Theorem~\ref{thm:godambe-alignment}).  Under identifiability with a single
data-generating distribution, the unconditional covariance reads
\begin{equation}
\Cov\!\bigl(g_B(\theta)\bigr)
\;=\;
\frac{1}{b}\,G(\theta),
\label{eq:align-intro}
\end{equation}
and, after projection onto the identifiable tangent space, we denote the
resulting population object $G^\star(\theta):=\Pi_\theta G(\theta)\Pi_\theta^\top$
(Section~\ref{sec:fisher-alignment}).  All downstream results are stated in
terms of $G^\star$.  For correctly specified likelihood losses,
$G^\star(\theta^\star)=F^\star(\theta^\star)$; the finite-dataset-reuse case
adds a $(1-b/n)$ correction and is treated separately
(Lemma~\ref{lem:finite_pop_cov_exact}, Proposition~\ref{prop:fpc_reuse}).

Once $\Gstar(\theta)$ is identified, three consequences follow within the
classical stochastic-approximation toolkit.  (a)~At fixed batch size,
constant-step SGD has a deterministic fluid limit (the gradient-flow
ODE), and its $\sqrt{b/\eta}$-scaled fluctuations obey a functional CLT
converging to a linear SDE with noise covariance $G^\star$; near a
nondegenerate optimum this is an OU process, and its stationary
covariance, multiplied back by $\eta/b$, is the discrete-time Lyapunov
benchmark of Proposition~\ref{prop:discrete_lyap}
(Section~\ref{sec:diffusion}).  (b)~Mean-square error in the
Fisher/identified metric admits a $1/N$ upper bound with an intrinsic
effective dimension $d_{\mathrm{eff}}$ and statistical condition number
$\kappa_F$ (Section~\ref{sec:rates}), matched in \emph{rate order} by an
i.i.d.\ parametric Fisher van~Trees lower bound
(Proposition~\ref{thm:lower-bound}); geometry dependence appears only in
the upper bound.  (c)~A high-probability oracle-complexity statement
follows, with a squared-gradient stationarity rate of order $N^{-1/2}$
(Theorem~\ref{thm:complexity_revised}).

\medskip
\noindent
{\setlength{\fboxsep}{6pt}%
\fbox{\parbox{\dimexpr\linewidth-2\fboxsep-2\fboxrule\relax}{\small%
\textbf{Main Results (informal).}\\[4pt]
\begin{tabular}{@{}r@{\;---\;}p{0.78\linewidth}@{}}
Theorem~\ref{thm:godambe-alignment} &
Covariance identification under exchangeable fresh sampling:
$\Cov(g_B\mid\mu)=b^{-1}G_\mu(\theta)$; under identifiability and after
projection onto the identifiable tangent space, the population
covariance is $\Cov(g_B)=b^{-1}G^\star(\theta)$.\\[3pt]
Theorem~\ref{thm:diffusion_limit} &
Fluid limit for the raw iterates and functional CLT for
$\sqrt{b/\eta}$-scaled fluctuations, with noise covariance
determined by $G^\star$.\\[3pt]
Thm.~\ref{thm:ou},\;Cor.~\ref{cor:lyapunov} &
OU equilibrium at $\theta^\star$ on the fluctuation scale; linearized
discrete-recursion stationary covariance on the original scale,
$\Sigma_\eta^{\mathrm{lin}}=(\eta/b)\Sigma_U+o(\eta/b)$
(Prop.~\ref{prop:discrete_lyap}).\\[3pt]
Thm.~\ref{thm:upper-bound},\;Prop.~\ref{thm:lower-bound} &
$1/(Tb)$ upper bound and rate-order-matching i.i.d.\ Fisher lower bound.\\[3pt]
Theorem~\ref{thm:complexity_revised} &
Squared-gradient stationarity at rate $N^{-1/2}$; norm stationarity at $N^{-1/4}$.
\end{tabular}%
}}}

\subsection{Organization}

Section~\ref{sec:lit} reviews related work.
Section~\ref{sec:setup} defines the oracle model and geometric objects.
Section~\ref{sec:fisher-alignment} derives the covariance identification from
exchangeability.
Section~\ref{sec:diffusion} builds the diffusion approximation, OU
linearization, and discrete-time bridge.
Section~\ref{sec:rates} proves an upper bound and a rate-order-matching
i.i.d.\ Fisher lower bound, and Section~\ref{sec:complexity} states oracle
complexity.
Section~\ref{sec:numerical} provides numerical validation, and
Section~\ref{sec:discussion} discusses extensions and practical implications.


\section{Literature Review and Positioning}
\label{sec:lit}

We organize prior work along two axes:
(i)~whether the mini-batch noise covariance is \emph{assumed}
(exogenous) or \emph{derived} from the sampling mechanism; and (ii)~whether
rates are stated in a Euclidean or in the identified statistical metric.
Table~\ref{tab:ontology-positioning} summarizes the resulting taxonomy.

\paragraph{Stochastic approximation and stationary covariance.}
Classical SA theory studies recursions
$\theta_{t+1}=\theta_t-\eta_t\{h(\theta_t)+\xi_{t+1}\}$
and derives a $\sqrt{t}$-CLT whose asymptotic covariance is governed by a Lyapunov equation
involving the Jacobian and a long-run noise covariance $\Gamma$
\citep{robbins1951stochastic,dvoretzky1956stochastic,kushner2003stochastic,polyak1992acceleration}.
This framework takes $\Gamma$ as given; the OU linearization, Lyapunov balance,
and $1/\sqrt{t}$ scaling are standard consequences.
Our contribution is logically prior: we derive $\Gamma=G^\star(\theta)/b$
from the mini-batch sampling design (Theorem~\ref{thm:godambe-alignment}),
converting the SA machinery from a conditional analysis into a structural one.

\paragraph{Diffusion approximations for SGD.}
A substantial literature interprets constant-step SGD as a discretization
of an SDE with temperature $\tau=\eta/b$ and various proposed diffusion
matrices $\mathbf{C}(\theta)$
\citep{mandt2017stochastic,li2017stochastic,smith2021origin}.  In those
works the diffusion covariance is either postulated, estimated from running
iterates, or justified heuristically through moment matching.  Our
covariance identification removes this modeling degree of freedom.  At
fixed batch size, we do not claim a nondegenerate diffusion limit of the
raw iterate path---its martingale variation is of order $\eta/b$ and
vanishes---but we prove a fluid limit plus a functional CLT on the
$\sqrt{b/\eta}$-rescaled fluctuations, with noise covariance
$C^\star{C^\star}^\top=G^\star$ (Theorem~\ref{thm:diffusion_limit}).  The
local OU regime and the Lyapunov balance then appear as the equilibrium
specialization on the fluctuation scale, with $\tau=\eta/b$ recovered on
the original iterate scale
(Theorem~\ref{thm:ou}, Corollary~\ref{cor:lyapunov}).  This is the
standard stochastic-approximation functional CLT architecture
\citep{kushner2003stochastic} applied to the fresh-sampling covariance
identified in Section~\ref{sec:fisher-alignment}, and is what fixes the
noise input to the subsequent rate analysis.

\paragraph{Scaling regimes.}
Recent rigorous diffusion-limit work studies joint regimes where $b$ grows
with $d$ \citep{paquette2024continuous}.
Our analysis is complementary and closer to OR operating constraints: we keep $d$
fixed and treat $b$ as a design variable subject to a budget, with particular
attention to the common regime $b\ll d$ where the \emph{matrix} geometry of
$G^\star(\theta)$ controls transient behavior and risk.

\paragraph{Information geometry, natural gradients, and Fisher-shaped preconditioning.}
Information geometry treats Fisher information as a Riemannian metric and
motivates natural-gradient updates $\theta_{t+1}=\theta_t-\eta\,F(\theta_t)^{-1}\nabla L(\theta_t)$
\citep{amari1998natural,martens2020new}.
Related methods---including stochastic gradient Fisher scoring
\citep{ahn2012bayesian} and KFAC \citep{martens2015optimizing}---inject or
exploit Fisher geometry exogenously in the drift or preconditioner.
Our contribution is distinct: even plain SGD, without any preconditioning,
inherits the identified statistical geometry endogenously through its noise covariance
under mini-batch sampling, changing the interpretation of diffusion strength, stationary
risk, and batch-size tuning without requiring natural-gradient updates.

\paragraph{Minimax rates and information-theoretic limits.}
Classical oracle lower bounds are stated in Euclidean norms under smoothness and
strong convexity \citep{nemirovski2009robust,agarwal2012information}, which may
not reflect the intrinsic difficulty of statistical objectives.
We give an upper bound in the identified statistical metric that surfaces
dependence on $\kappa_F$ and $d_{\eff}$
(Theorem~\ref{thm:upper-bound}), and match its $1/N$ rate with an i.i.d.\
van~Trees lower bound in the parametric Fisher setting
(Proposition~\ref{thm:lower-bound}); the lower bound is at the level of
rate order and does not certify the geometry dependence.

\paragraph{OR context: sampling budgets.}
In simulation optimization and stochastic programming, gradients are estimated
from sampled scenarios and sampling cost is explicit
\citep{fu2006gradient,shapiro2014lectures,chernoff1959sequential}.
Because batching controls temperature $\tau=\eta/b$, it directly shapes
diffusion amplitude and terminal risk; Section~\ref{sec:discussion} outlines
the resulting effort allocation problem.

\paragraph{Why small batches can work when $b\ll d$.}
The empirical preference for small batches
\citep{keskar2017large,shallue2019measuring,mccandlish2018empirical} is
explained by our framework: smaller $b$ buys more contraction steps under fixed
budgets, the identified covariance concentrates noise along statistically flat
directions, and rates scale with $d_{\eff}\ll d$.

\begin{table}
\TABLE
{Positioning: primitives, contrasts, and where each appears in the paper. \label{tab:ontology-positioning}}
{\begin{tabular}{@{}p{2.2cm}@{\quad}p{3.2cm}@{\quad}p{4.0cm}@{\quad}p{2.8cm}@{}}
\hline\up
Primitive & Status quo & This paper & Pointers \\
\hline\up
Mini-batch noise covariance &
Exogenous / isotropic / heuristic $\Gamma$ &
Identified from sampling: $\Cov(g_B)\propto b^{-1}G^\star(\theta)$; Fisher is a special case &
Sec.~\ref{sec:fisher-alignment};
Thm.~\ref{thm:godambe-alignment} \\[4pt]
Continuous-time limit &
Postulated SDE with ad hoc $\mathbf C(\theta)$ &
Fluid limit (ODE) for raw iterates; functional CLT for $\sqrt{b/\eta}$-scaled fluctuations with diffusion $G^\star(\theta)$ &
Sec.~\ref{sec:diffusion};
Thm.~\ref{thm:diffusion_limit} \\[4pt]
OU / Lyapunov equilibrium &
Heuristic OU with ad hoc covariance &
OU fluctuation limit at $\theta^\star$; Lyapunov balance
$H^\star\Sigma_U+\Sigma_U H^{\star\top}=G^\star$ on fluctuation scale;
linearized original-scale benchmark $\Sigma_\eta^{\mathrm{lin}}=(\eta/b)\Sigma_U+o(\eta/b)$ (Prop.~\ref{prop:discrete_lyap} / Cor.~\ref{cor:lyapunov}) &
Sec.~\ref{sec:diffusion};
Cor.~\ref{cor:lyapunov} \\[4pt]
Metric for rates &
Euclidean distance; $\kappa_H$ dominates &
Identified statistical metric is the natural risk &
Sec.~\ref{sec:rates};
Thm.~\ref{thm:upper-bound} \\[4pt]
Lower bounds &
Euclidean oracle lower bounds &
Information-theoretic lower bound in Fisher metric (van Trees) &
Sec.~\ref{sec:rates};
Prop.~\ref{thm:lower-bound} \\[4pt]
Dimension notion &
Ambient $d$ &
Effective dimension $d_{\eff}$ induced by $F^\star$ / $G^\star$ &
Sec.~\ref{sec:rates}; Def.~\eqref{eq:deff} \\[4pt]
Oracle complexity &
Euclidean stationarity; $\kappa_H$ &
Fisher-dual stationarity; $\kappa_F$ and $d_{\eff}$ &
Sec.~\ref{sec:complexity};
Thm.~\ref{thm:complexity_revised} \down\\
\hline
\end{tabular}}{}
\end{table}


\section{Problem Setup and Preliminaries}
\label{sec:setup}

This section fixes notation, specifies the stochastic-gradient oracle, and
records the geometric objects that govern the diffusion approximation and the
rate/complexity bounds.  The perspective is \emph{Mathematics of Operations
Research}: we make the sampling design explicit, separate intrinsic noise
geometry from finite-population correction, and state the local regularity
conditions used throughout.

\subsection{Notation and macros used throughout}
\label{subsec:setup_notation}

We work on a measurable space $(\cZ,\cA)$.
Vectors use the Euclidean inner product $\langle u,v\rangle=u^\top v$ and norm
$\|u\|_2$.  For a symmetric matrix $A$, $\|A\|_{\op}$ denotes the operator norm
and $\lambda_{\min}(A),\lambda_{\max}(A)$ its extreme eigenvalues.

\subsection{Data model, objective, and stochastic-gradient oracle}
\label{subsec:setup_oracle}

Let $\Theta\subseteq\R^d$ be open and convex. Let $Q$ denote the (unknown) data
distribution on $(\cZ,\cA)$, and let $\ell:\Theta\times\cZ\to\R$ be a measurable
loss. Define the population risk
\begin{equation}
L(\theta)
\;:=\;
\E_{Z\sim Q}\bigl[\ell(\theta;Z)\bigr],
\label{eq:def_L_setup}
\end{equation}
and let $\theta^\star\in\arg\min_{\theta\in\Theta}L(\theta)$ be a (local) minimizer.

\paragraph{Finite-population sampling model.}
We observe a dataset $z_{1:n}:=(z_1,\dots,z_n)$ drawn i.i.d.\ from $Q$ (or, in the
exchangeability formulation of Section~\ref{sec:fisher-alignment}, conditionally
i.i.d.\ given a directing measure).
The lower bound in Section~\ref{sec:rates} is stated for the i.i.d.\
parametric Fisher setting; adaptive-oracle extensions are discussed in
Remark~\ref{rem:beyond_iid}.
At iteration $t$, the algorithm selects a
mini-batch $B_t\subseteq\{1,\dots,n\}$ of size $b_t$ uniformly \emph{without
replacement} and forms the mini-batch gradient
\begin{equation}
g_t(\theta)
\;:=\;
\frac{1}{b_t}\sum_{i\in B_t}\psi(\theta;z_i),
\qquad
\psi(\theta;z):=\nabla_\theta \ell(\theta;z).
\label{eq:def_gt_setup}
\end{equation}
The SGD recursion is
\begin{equation}
\theta_{t+1}
\;=\;
\theta_t-\eta_t\,g_t(\theta_t),
\label{eq:sgd_setup}
\end{equation}
with stepsize $\eta_t>0$.

\paragraph{Filtration and martingale structure.}
Let $\cF_t$ be the $\sigma$-field generated by $(\theta_0,B_0,\dots,B_{t-1},\theta_t)$
and the dataset $z_{1:n}$. Then
\[
\xi_t(\theta_t):=g_t(\theta_t)-\E[g_t(\theta_t)\mid \cF_t]
\]
is a martingale difference sequence with respect to $(\cF_t)$, and all diffusion
limits are driven by the predictable quadratic variation of $\sum_t \eta_t\xi_t$.

\paragraph{Cost accounting (oracle calls).}
We measure computational cost in \emph{sample-gradient evaluations}. One iteration
with batch size $b_t$ uses $b_t$ sample evaluations. After $T$ iterations, the
total sampling budget is
\begin{equation}
N_T
\;:=\;
\sum_{t=0}^{T-1} b_t.
\label{eq:NT_def}
\end{equation}
We report results both in iterations $T$ and in oracle calls $N_T$, since
control/OR questions naturally optimize over batch
allocation under a fixed $N_T$.

\subsection{Likelihood losses: Fisher information and the induced metric}
\label{subsec:setup_fisher}

When $\ell(\theta;z)=-\log p_\theta(z)$ for a parametric family
$\{P_\theta:\theta\in\Theta\}$ with density $p_\theta$, the score is
$s_\theta(z):=\nabla_\theta\log p_\theta(z)$ and the loss gradient is the
negative score, $\nabla_\theta \ell(\theta;z)=-s_\theta(z)$. The intrinsic noise
geometry reduces to Fisher information at the correctly specified point $\theta^\star$.

\begin{definition}[Fisher information]
\label{def:fisher_setup}
Assume $p_\theta$ is differentiable in $\theta$ and $s_\theta(Z)$ is square
integrable under $P_\theta$. The Fisher information is
\begin{equation}
F(\theta)
\;:=\;
\E_\theta\!\bigl[s_\theta(Z)s_\theta(Z)^\top\bigr].
\label{eq:def_fisher_setup}
\end{equation}
\end{definition}

\begin{lemma}[Score identities {\citep[Ch.~7]{lehmann1998theory}}]
\label{lem:score_identities_setup}
Assume $\int p_\theta(z)\,dz=1$ and differentiation under the integral is valid.
Then (i) $\E_\theta[s_\theta(Z)]=0$ and (ii)
\begin{equation}
F(\theta)
\;=\;
-\E_\theta\!\bigl[\nabla_\theta^2\log p_\theta(Z)\bigr].
\label{eq:info_equality_setup}
\end{equation}
\end{lemma}

\begin{proof}{Proof.}
Standard; see, e.g., \citet[Ch.~7]{lehmann1998theory}. \Halmos
\end{proof}

\paragraph{Fisher metric.}
If $F(\theta)\succ 0$, it defines the local inner product
$\langle u,v\rangle_{F(\theta)}:=u^\top F(\theta)v$ and norm
$\|u\|_{F(\theta)}:=\sqrt{u^\top F(\theta)u}$.
On compact sets where $F(\theta)$ is bounded and bounded away from $0$, Fisher and
Euclidean norms are equivalent.

\subsection{General losses: sandwich / Godambe geometry}
\label{subsec:setup_godambe}

For a general loss, information equality need not hold. The relevant local geometry
splits into (i) curvature of the drift and (ii) covariance of the estimating equation (gradient).

\begin{definition}[Noise covariance, sensitivity, and Godambe information]
\label{def:godambe_setup}
Let $\psi(\theta;z):=\nabla_\theta \ell(\theta;z)$ and define
\begin{align}
G(\theta) &:= \Cov\bigl(\psi(\theta;Z)\bigr),
\label{eq:def_G_setup}\\
H(\theta) &:= \nabla^2 L(\theta)
          \;=\;\E\!\left[\nabla_\theta \psi(\theta;Z)\right],
\label{eq:def_H_setup}
\end{align}
where expectations are under $Z\sim Q$. When $H(\theta)$ is invertible on the
identifiable directions, the classical sandwich covariance and Godambe information are
\begin{equation}
\Sigma_{\mathrm{sand}}(\theta)
:=H(\theta)^{-1}G(\theta)H(\theta)^{-1},
\qquad
J_{\mathrm{God}}(\theta):=\Sigma_{\mathrm{sand}}(\theta)^{-1}
=H(\theta)^\top G(\theta)^{-1}H(\theta).
\label{eq:def_godambe_setup}
\end{equation}
\end{definition}

\begin{remark}[Likelihood as a special case]
If $\ell(\theta;z)=-\log p_\theta(z)$ and the model is correctly specified, then
$\psi=s_\theta$ and $G(\theta)=H(\theta)=F(\theta)$ at $\theta^\star$, so the
sandwich collapses and the Godambe geometry equals Fisher.
\end{remark}

\subsection{Projection onto identifiable directions}
\label{subsec:setup_projection}

Projection is only needed when the parametrization is redundant or when only a
lower-dimensional tangent space is identifiable.

\begin{definition}[Tangent-space projection and projected geometry]
\label{def:projection_setup}
Let $T_\theta\subseteq\R^d$ denote the identifiable tangent space at $\theta$
(i.e., the linear span of directions along which the loss varies at first order;
formally $T_\theta:=\mathrm{range}(H(\theta))$, or equivalently
$\mathrm{supp}(F(\theta))$ in the likelihood case), and
let $\Pi_\theta$ be the Euclidean-orthogonal projector onto $T_\theta$.
Define the projected noise covariance
\begin{equation}
G^\star(\theta)
\;:=\;
\Pi_\theta\,G(\theta)\,\Pi_\theta^\top,
\label{eq:def_Gstar_setup}
\end{equation}
and, in the likelihood case, the projected Fisher information
\begin{equation}
F^\star(\theta)
\;:=\;
\Pi_\theta\,F(\theta)\,\Pi_\theta^\top.
\label{eq:def_Fstar_setup}
\end{equation}
When $T_\theta=\R^d$, $\Pi_\theta=I$ and $G^\star=G$, $F^\star=F$.
\end{definition}

\begin{remark}[Terminological convention]
\label{rem:terminology}
Throughout the paper, the identified noise matrix is $G^\star(\theta)$ (the
projected gradient covariance).  In the correctly specified likelihood case,
$G^\star(\theta^\star)=F^\star(\theta^\star)$ (projected Fisher information).
For general losses, $G^\star(\theta^\star)$ is \emph{not} itself the Godambe
information; rather, paired with the Hessian $H^\star$ it induces the
classical sandwich covariance $H^{\star-1}G^\star H^{\star-1}$ and, by
inversion, the Godambe information $J_{\mathrm{God}}=H^{\star\top}G^{\star-1}H^\star$.
Every theorem statement holds with $G^\star$ in the general case and
specializes to $F^\star$ under correct specification. When we write ``Fisher
metric'' or ``Fisher-dual norm,'' the general-loss reader should substitute the
corresponding object defined by $G^\star$; the paired Godambe/sandwich
interpretation always requires $H^\star$ as well.
\end{remark}

\begin{remark}[What enters later sections]
Sections~\ref{sec:fisher-alignment}--\ref{sec:diffusion} show that the mini-batch
noise covariance is (up to a sampling-design factor) exactly $G^\star(\theta)/b_t$
and reduces to $F^\star(\theta)/b_t$ under correct specification.  This is the
only matrix that enters the diffusion coefficient and Lyapunov balance.
\end{remark}

\subsection{Local regularity conditions}
\label{subsec:setup_assumptions}

All results are local near $\theta^\star$ and use standard smoothness/moment and
nondegeneracy conditions.

\begin{assumption}[Smoothness and moments]
\label{ass:regularity_setup}
The loss $\ell(\theta;z)$ is twice continuously differentiable in $\theta$.
There exist constants $M<\infty$ and $m_4<\infty$ such that for all $\theta$ in a
neighborhood $\mathcal{N}$ of $\theta^\star$:
\begin{enumerate}[label=(\alph*)]
\item $\|\nabla_\theta^2 \ell(\theta;z)\|_{\op}\le M$ almost surely;
\item $\E\bigl[\|\psi(\theta;Z)\|_2^4\bigr]\le m_4$.
\end{enumerate}
\end{assumption}

\begin{assumption}[Local stability]
\label{ass:strong_convex_setup}
The population risk is locally strongly convex at $\theta^\star$:
\begin{equation}
H^\star:=\nabla^2 L(\theta^\star)\succeq \mu I
\quad\text{for some }\mu>0.
\label{eq:local_strong_convex_setup}
\end{equation}
\end{assumption}

\begin{assumption}[Geometric nondegeneracy]
\label{ass:geom_nondegen_setup}
The intrinsic noise geometry is positive definite on identifiable directions:
\begin{equation}
G^\star(\theta^\star)\succeq \lambda_{\min} I
\quad\text{for some }\lambda_{\min}>0.
\label{eq:geom_nondegen_setup}
\end{equation}
(In the correctly specified likelihood case this is $F^\star(\theta^\star)\succ 0$.)
\end{assumption}

\begin{remark}[Role of assumptions]
Assumptions~\ref{ass:regularity_setup}--\ref{ass:geom_nondegen_setup}
provide Taylor/moment control, a locally stable equilibrium, and
nondegeneracy of the diffusion coefficient, respectively.
\end{remark}

\subsection{Key geometric parameters (preview)}
\label{subsec:setup_keyparams}

Several quantities derived from the local geometry at $\theta^\star$ govern
the rate and complexity results of
Sections~\ref{sec:rates}--\ref{sec:complexity}.  We collect their definitions
here so the reader can refer to them before encountering the formal theorems.

\begin{definition}[Fisher-strong-convexity constant]
\label{def:muF_setup}
Let $H^\star:=\nabla^2 L(\theta^\star)$ and $F^\star:=F(\theta^\star)$ (or
$G^\star(\theta^\star)$ for general losses).  Define
\begin{equation}
\mu_F
\;:=\;
\lambda_{\min}\!\left(
  \frac{(F^\star)^{1/2}H^\star(F^\star)^{-1/2}
        + (F^\star)^{-1/2}H^\star(F^\star)^{1/2}}{2}
\right).
\label{eq:muF_setup}
\end{equation}
Equivalently, $\mu_F$ is the smallest eigenvalue of the symmetrized
curvature--information interaction
$\mathrm{Sym}((F^\star)^{1/2}H^\star(F^\star)^{-1/2})$.
\end{definition}

\begin{definition}[Fisher condition number and effective dimension]
\label{def:kappaF_deff_setup}
\begin{equation}
\kappa_F
\;:=\;
\frac{\lambda_{\max}(F^\star)}{\lambda_{\min}(F^\star)},
\qquad
d_{\mathrm{eff}}(F^\star)
\;:=\;
\frac{\Tr(F^\star)}{\lambda_{\max}(F^\star)}
\;\in\;(0,d].
\label{eq:kappaF_deff_setup}
\end{equation}
\end{definition}

The quantity $\kappa_F$ measures statistical anisotropy rather than
Euclidean ill-conditioning; a problem can be Euclidean-stiff
($\kappa_H\gg 1$) yet statistically well-conditioned ($\kappa_F=O(1)$)
when curvature and information share eigenstructure.
The effective dimension $d_{\mathrm{eff}}$ is the stable rank of $F^\star$:
it equals $d$ when $F^\star$ is isotropic and approaches~$1$ when a single
direction dominates.

\begin{remark}[Generalized eigenvalue interpretation]
\label{rem:gen_eig}
The constant $\mu_F$ is linked to the generalized eigenvalues of the pair
$(H^\star, F^\star)$.  If $H^\star$ and $F^\star$ commute (i.e., share an
eigenbasis $\{v_j\}$ with respective eigenvalues $h_j$ and $f_j$), then
$(F^\star)^{1/2}H^\star(F^\star)^{-1/2}$ has eigenvalues
$f_j^{1/2}\,h_j\,f_j^{-1/2} = h_j$, so the symmetrization is trivial and
$\mu_F = \min_j h_j = \lambda_{\min}(H^\star)$: the Fisher-strong-convexity
constant reduces to the ordinary strong-convexity constant.
In the non-commuting case, $(F^\star)^{1/2}H^\star(F^\star)^{-1/2}$ is
non-normal (its eigenvalues are still those of $H^\star$, but it is no
longer symmetric), and the minimum of its quadratic form---i.e.,
$\mu_F$---can drop strictly below $\lambda_{\min}(H^\star)$.
Geometrically, this occurs when the Fisher metric amplifies directions
along which curvature is weak, making the effective contraction rate
slower than the Euclidean strong-convexity constant.  The product $\kappa_F \cdot d_{\mathrm{eff}}$ that enters
the oracle complexity (Theorem~\ref{thm:complexity_revised}) therefore
measures the total cost of resolving all statistically relevant directions
at the hardest direction's scale.
\end{remark}

\begin{table}[t]
\TABLE
{Key notation. All quantities are evaluated locally at or near $\theta^\star$ unless noted otherwise.\label{tab:notation}}
{\begin{tabular}{@{}ll@{}}
\hline\up
Symbol & Meaning \\
\hline\up
$L(\theta)=\E[\ell(\theta;Z)]$ & Population risk (loss) \\
$\psi(\theta;z)=\nabla_\theta\ell(\theta;z)$ & Per-sample estimating equation (gradient) \\
$g_B(\theta)$ & Mini-batch gradient estimator \\
$G(\theta)=\Cov(\psi(\theta;Z))$ & Gradient covariance; $G^\star=\Pi G\Pi^\top$ projected \\
$F(\theta)=\E[s_\theta s_\theta^\top]$ & Fisher information (likelihood case) \\
$F^\star(\theta^\star)$ & Projected Fisher; equals $G^\star(\theta^\star)$ under correct specification \\
$H^\star=\nabla^2 L(\theta^\star)$ & Hessian at optimum (drift curvature) \\
$J_{\mathrm{God}}=H^\top G^{-1}H$ & Godambe (sandwich) information \\
$\tau=\eta/b$ & Effective temperature \\
$\mu_F$ & Fisher-strong-convexity constant (Def.~\ref{def:muF_setup}) \\
$\kappa_F$ & Fisher condition number (Def.~\ref{def:kappaF_deff_setup}) \\
$d_{\eff}(F^\star)$ & Effective dimension / stable rank of $F^\star$ \down\\
\hline
\end{tabular}}{}
\end{table}


\section{Covariance Identification for Mini-Batch Noise}
\label{sec:fisher-alignment}

Mini-batch gradient noise is not an arbitrary covariance: it is determined by (i) the sampling design used to form
mini-batches and (ii) the intrinsic covariance of the per-sample estimating
equation $\psi(\theta;Z)$ (Fisher information in the well-specified likelihood
case, the projected gradient covariance more generally).  We make this precise
in two steps:

\begin{itemize}
\item an \emph{exact} finite-population identity conditional on the realized
dataset (classical survey-sampling algebra), and
\item an \emph{exchangeable/i.i.d.} limit that converts the dataset-dependent
quantity into the population covariance $G(\theta):=\Cov(\psi(\theta;Z))$ (or its
projected analogue).
\end{itemize}

\subsection{Finite-population covariance identity (exact, conditional on data)}
\label{subsec:finite-pop}

Fix $\theta$ and condition on the realized dataset $\{z_1,\dots,z_n\}$.  Write
$\psi_i:=\psi(\theta;z_i)$ and define the finite-population mean and covariance
\[
\bar\psi_n(\theta):=\frac{1}{n}\sum_{i=1}^n\psi_i,
\qquad
S_n(\theta):=\frac{1}{n-1}\sum_{i=1}^n(\psi_i-\bar\psi_n)(\psi_i-\bar\psi_n)^\top.
\]
Let $B$ be a simple random sample without replacement of size $b$ from
$\{1,\dots,n\}$ and define the mini-batch average
\[
g_B(\theta):=\frac{1}{b}\sum_{i\in B}\psi_i .
\]

\begin{lemma}[Exact finite-population covariance of the mini-batch mean]
\label{lem:finite_pop_cov_exact}
Conditionally on the dataset,
\begin{equation}
\Cov\!\bigl(g_B(\theta)\mid z_{1:n}\bigr)
\;=\;
\frac{1}{b}\Bigl(1-\frac{b}{n}\Bigr)\,S_n(\theta).
\label{eq:finite_pop_exact}
\end{equation}
\end{lemma}

\begin{proof}{Proof.}
Write $g_B(\theta)=\frac{1}{b}\sum_{i=1}^n I_i\,\psi_i$ where $I_i=\mathbf 1\{i\in B\}$.
Under simple random sampling without replacement, $\E[I_i]=b/n$,
$\Var(I_i)=\frac{b}{n}(1-\frac{b}{n})$, and for $i\neq j$,
$\Cov(I_i,I_j)=-\frac{b(n-b)}{n^2(n-1)}$.  Expanding
$\Cov(\sum_i I_i\psi_i\mid z_{1:n})$ and collecting terms yields
\eqref{eq:finite_pop_exact}. \Halmos
\end{proof}

\begin{remark}[Equivalent correction-factor forms]
\label{rem:fpc_forms}
Using $S_n$ with the $1/(n-1)$ normalization gives the factor $(1-b/n)$.
If instead one uses the $1/n$ covariance normalization, the factor becomes
$1-(b-1)/(n-1)$.  The two are algebraically equivalent up to the change of
normalization.
\end{remark}

Lemma~\ref{lem:finite_pop_cov_exact} is the clean MOR-style decomposition:
\emph{sampling design} contributes the scalar finite-population correction
$1-b/n$, while the \emph{geometry} lives in $S_n(\theta)$.

\subsection{From exchangeability to the population covariance}
\label{subsec:exchangeability-to-geometry}

We now convert the dataset-dependent covariance $S_n(\theta)$ into a
population-level object.

Let $(Z_i)_{i\ge 1}$ be exchangeable with de Finetti directing measure $\mu$
\citep{hewitt1955symmetric}.
Conditional on $\mu$, the sequence is i.i.d. with law $P_\mu$.  Define the
population covariance (conditional on $\mu$)
\begin{equation}
G_\mu(\theta)\;:=\;\Cov_\mu\bigl(\psi(\theta;Z)\bigr),
\qquad Z\sim P_\mu,
\label{eq:Gmu_def}
\end{equation}
and, when identifiability requires it, the projected version
$G_\mu^\star(\theta):=\Pi_\theta G_\mu(\theta)\Pi_\theta^\top$.

\begin{theorem}[Covariance identification under exchangeability]
\label{thm:godambe-alignment}
Assume $\E_\mu\|\psi(\theta;Z)\|_2^2<\infty$.
Let $g_B(\theta)=\frac{1}{b}\sum_{i\in B}\psi(\theta;Z_i)$ where $B$ is a uniform
subset of $\{1,\dots,n\}$ of size $b$ (sampling without replacement).
Then, conditional on $\mu$,
\begin{equation}
\Cov\!\bigl(g_B(\theta)\mid \mu\bigr)
\;=\;
\frac{1}{b}\,G_\mu(\theta).
\label{eq:godambe_alignment_exact}
\end{equation}
When the conditional mean $\E[\psi(\theta;Z)\mid\mu]$ does not vary with
$\mu$ (identifiability under a single data-generating distribution), the
unconditional covariance reads $\Cov(g_B(\theta))=b^{-1}G(\theta)$ where
$G(\theta):=\E_\mu[G_\mu(\theta)]$; after projection onto the identifiable
tangent space (Section~\ref{sec:setup}), we denote this
$G^\star(\theta)=\Pi_\theta G(\theta)\Pi_\theta^\top$.  The remainder of
the paper works with $G^\star$.
\emph{Interpretation.}  Equation~\eqref{eq:godambe_alignment_exact} is exact
(not an approximation) when one marginalizes over both the mini-batch index $B$
\emph{and} the dataset realization conditional on $\mu$.
The finite-population correction $(1-b/n)$ of Lemma~\ref{lem:finite_pop_cov_exact}
is present when one conditions on a fixed dataset, but it disappears at the
population level because the randomness of the dataset itself contributes a
complementary term through the total-covariance decomposition.
Equation~\eqref{eq:conditional_vs_population} below makes this bookkeeping
explicit.

Moreover,
\begin{equation}
\E\!\Big[\Cov(g_B(\theta)\mid Z_{1:n},\mu)\,\big|\,\mu\Big]
\;=\;
\frac{1}{b}\Bigl(1-\frac{b}{n}\Bigr)\,\E[S_n(\theta)\mid\mu]
\;=\;
\frac{1}{b}\Bigl(1-\frac{b}{n}\Bigr)\,G_\mu(\theta),
\label{eq:conditional_vs_population}
\end{equation}
and the remaining part of \eqref{eq:godambe_alignment_exact} comes from the
randomness of the dataset through the total-covariance decomposition.
\end{theorem}

\begin{proof}{Proof.}
Conditional on $\mu$, the sequence $(Z_i)_{i\ge 1}$ is i.i.d.\ with law
$P_\mu$; therefore any uniformly sampled subset of size $b$ has the same joint
distribution as the first $b$ coordinates.
Hence selecting indices $B$ without replacement from $\{1,\dots,n\}$ is
distributionally equivalent to taking the first $b$ terms, and $g_B(\theta)$ is
the average of $b$ i.i.d. draws from $P_\mu$. This yields
$\Cov(g_B(\theta)\mid\mu)=\frac{1}{b}G_\mu(\theta)$, proving
\eqref{eq:godambe_alignment_exact}.

For \eqref{eq:conditional_vs_population}, apply the exact identity
Lemma~\ref{lem:finite_pop_cov_exact} conditional on $(Z_{1:n},\mu)$ and take
$\E[\cdot\mid \mu]$. Under $\E_\mu\|\psi\|^2<\infty$,
$\E[S_n(\theta)\mid\mu]=G_\mu(\theta)$ (the $1/(n-1)$ sample-covariance normalization is unbiased).
\Halmos
\end{proof}

\begin{remark}[What ``alignment'' means operationally]
\label{rem:alignment_operational}
Equation \eqref{eq:godambe_alignment_exact} is the alignment statement used by
the diffusion/Lyapunov analysis: \emph{mini-batch noise has covariance
$\frac{1}{b}$ times a population matrix determined by the loss and the data
law}.  The finite-population correction in Lemma~\ref{lem:finite_pop_cov_exact}
matters when one conditions on a fixed dataset and reuses it many times; in the
standard stochastic-approximation regime with fresh randomness across
iterations, the population covariance \eqref{eq:godambe_alignment_exact} is the
relevant object.
\end{remark}

\begin{proposition}[Finite-population correction under dataset reuse]
\label{prop:fpc_reuse}
Under the notation of Theorem~\ref{thm:godambe-alignment}, three
conditioning levels yield distinct covariance expressions:
\begin{enumerate}[label=(\roman*),leftmargin=2em]
\item \emph{Conditional on the dataset $z_{1:n}$} (fixed pool, sampling without
      replacement):
      $\Cov(g_B(\theta)\mid z_{1:n}) = \frac{1}{b}(1-\frac{b}{n})\,S_n(\theta)$
      (Lemma~\ref{lem:finite_pop_cov_exact}).
\item \emph{Conditional on the directing measure $\mu$} (fresh data per
      iteration):
      $\Cov(g_B(\theta)\mid\mu) = \frac{1}{b}\,G_\mu(\theta)$
      (Theorem~\ref{thm:godambe-alignment}), with no finite-population
      correction.
\item \emph{Population level} (unconditional):
      $\Cov(g_B(\theta)) = \frac{1}{b}\,\E_\mu[G_\mu(\theta)] + \Cov_\mu(\E[g_B(\theta)\mid\mu])$.
      This reduces to $\frac{1}{b}\,G(\theta)$ when the conditional mean $\E[\psi(\theta;Z)\mid\mu]$ does not vary with $\mu$ (e.g., under identifiability with a single data-generating distribution).
\end{enumerate}
When a finite dataset is reused across many iterations (i.e., without-replacement
sampling from the same $z_{1:n}$ at every step), the conditional
covariance carries the factor $(1-b/n)$, which is non-negligible unless
$b=o(n)$.  The main diffusion, rate, and complexity results
(Sections~\ref{sec:diffusion}--\ref{sec:complexity}) are population-level
or fresh-sampling statements that use level~(ii).
\end{proposition}

\begin{proof}{Proof.}
Level~(i) is Lemma~\ref{lem:finite_pop_cov_exact}.  Level~(ii) is
Theorem~\ref{thm:godambe-alignment}.  Level~(iii) follows by integrating
level~(ii) over $\mu$ under identifiability.  The total-covariance
decomposition
$\Cov(g_B) = \E_\mu[\Cov(g_B\mid\mu)] + \Cov_\mu(\E[g_B\mid\mu])$
reconciles levels~(ii) and~(iii), with the second term contributing when
$\E[g_B\mid\mu]$ varies with $\mu$. \Halmos
\end{proof}

\begin{remark}[Batch size rescales temperature, not geometry]
\label{rem:batch_temperature}
Equation~\eqref{eq:godambe_alignment_exact} implies that increasing the
batch size $b$ rescales the effective temperature
$\tau=\eta/b$ (Section~\ref{sec:diffusion}) but does \emph{not} alter the
noise ellipsoid: its shape is entirely determined by the statistical
geometry $G_\mu(\theta)$.
This decoupling---scale from the sampling budget, shape from the
loss---underlies the subsequent rate and complexity results.
\end{remark}

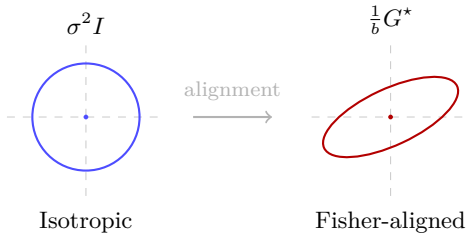
\begin{figure}[t]
\centering
\begin{tikzpicture}[scale=0.75]
\begin{scope}[shift={(-2.7,0)}]
  \draw[gray!40, dashed] (-1.4,0)--(1.4,0);
  \draw[gray!40, dashed] (0,-1.4)--(0,1.4);
  \draw[black, thick, dashed] (0,0) circle (0.95);
  \fill[black] (0,0) circle (1.2pt);
  \node[below] at (0,-1.5) {\footnotesize Isotropic};
  \node[above] at (0,1.3) {\footnotesize $\sigma^2 I$};
\end{scope}
\begin{scope}[shift={(2.7,0)}]
  \draw[gray!40, dashed] (-1.4,0)--(1.4,0);
  \draw[gray!40, dashed] (0,-1.4)--(0,1.4);
  \draw[black, thick, rotate=25]
    (0,0) ellipse (1.3 and 0.5);
  \fill[black] (0,0) circle (1.2pt);
  \node[below] at (0,-1.5) {\footnotesize Fisher-aligned};
  \node[above] at (0,1.3) {\footnotesize $\tfrac{1}{b}G^\star$};
\end{scope}
\draw[->, thick, gray!60] (-0.8,0) -- (0.6,0);
\node[above, gray!60] at (-0.1,0.12) {\scriptsize alignment};
\end{tikzpicture}
\caption{Isotropic noise (left, dashed circle) vs.\ Fisher-aligned noise
(right, solid ellipse) at the same total variance.  Batch size rescales
temperature but not shape.}
\label{fig:noise_schematic}
\end{figure}

\begin{remark}[Specialization to Fisher alignment]
\label{rem:fisher_corollary}
If $\ell(\theta;z)=-\log p_\theta(z)$ and the model is correctly specified, then
$\psi(\theta;z)=s_\theta(z)$ is the score and the Bartlett identity yields
$\Cov(s_{\theta}(Z))=F(\theta)$ at $\theta=\theta^\star$.  Projecting onto the
identifiable tangent space gives $G_\mu^\star(\theta^\star)=F^\star(\theta^\star)$.
Thus \eqref{eq:godambe_alignment_exact} becomes
$\Cov(g_B(\theta^\star)\mid\mu)=\frac{1}{b}F^\star(\theta^\star)$, i.e., Fisher alignment.
\end{remark}

\begin{remark}[No Gaussian assumption]
Neither Lemma~\ref{lem:finite_pop_cov_exact} nor
Theorem~\ref{thm:godambe-alignment} assumes Gaussian noise.  Gaussianity enters
later only as a \emph{diffusion approximation} (martingale CLT) when we pass from
the discrete recursion to an OU limit.
\end{remark}


\section{Diffusion Limits: Fluid Scale, Fluctuation Scale, and the OU Regime}
\label{sec:diffusion}

This section links the discrete-time SGD recursion to a continuous-time
stochastic process whose diffusion coefficient is determined by the
mini-batch covariance $G^\star(\theta)$ identified in
Section~\ref{sec:fisher-alignment}.  The link is \emph{not} a raw diffusion
limit at fixed batch size: the martingale quadratic variation of the raw
interpolation is of order $\eta/b$ and vanishes as $\eta\downarrow 0$, so
the raw process converges to the deterministic gradient-flow ODE.  The
nondegenerate stochastic limit appears only after centering around that
ODE and rescaling by $\sqrt{b/\eta}$, yielding a fluctuation process whose
limiting covariance structure is exactly $G^\star$.  Near a nondegenerate optimum, this fluctuation process is an
Ornstein--Uhlenbeck (OU) process whose stationary covariance $\Sigma_U$
satisfies a fluctuation-scale Lyapunov equation.  On the original iterate
scale, the linearized discrete recursion of
Proposition~\ref{prop:discrete_lyap} admits a unique stationary covariance
$\Sigma_\eta^{\mathrm{lin}}$ that matches $(\eta/b)\Sigma_U$ at leading
order in $\eta$, recovering the paper's ``effective temperature''
$\tau=\eta/b$ on the correct scale; we do not claim an invariant law for
the nonlinear SGD iterates.  Diffusion approximations are a standard tool in OR for
reducing stochastic systems to tractable Brownian models
\citep{harrison1985brownian,whitt2002stochastic}; the present setting is
analogous, with the sampling mechanism playing the role of the arrival
process.

Throughout this section we work in the \emph{fresh-sampling / population}
conditioning regime of Theorem~\ref{thm:godambe-alignment}: at each iterate
the mini-batch is a simple random sample drawn with fresh randomness from
the population (equivalently, conditional on the de~Finetti directing
measure $\mu$).  The fixed-dataset reuse case with $(1-b/n)$ correction
belongs to Lemma~\ref{lem:finite_pop_cov_exact} and
Proposition~\ref{prop:fpc_reuse} and is not used below.

\subsection{SGD recursion and the martingale-noise decomposition}
\label{subsec:diffusion_decomp}

Fix a batch size $b\in\N$ and a step size $\eta>0$.  Let
\[
h(\theta):=\nabla L(\theta)=\E[\psiFn(\theta;Z)],
\qquad
\psiFn(\theta;z):=\nabla_\theta \ell(\theta;z),
\]
and write the mini-batch gradient estimator as
$g_t(\theta_t)=h(\theta_t)+\xi_t(\theta_t)$ with
$\E[\xi_t(\theta_t)\mid\cF_t]=0$, where $\cF_t$ is the natural filtration.
The SGD recursion is
\begin{equation}
\theta_{k+1}^\eta=\theta_k^\eta-\eta\,h(\theta_k^\eta)-\eta\,\xi_{k+1}^\eta,
\label{eq:sgd_decomp}
\end{equation}
with the superscript $\eta$ emphasizing the dependence of the iterate path
on the step size.  Under the fresh-sampling conditioning,
\begin{equation}
\E\!\bigl[\xi_{k+1}^\eta(\xi_{k+1}^\eta)^\top\mid\cF_k^\eta\bigr]
\;=\;\frac{1}{b}\,G^\star(\theta_k^\eta)+r_{k+1}^\eta,
\label{eq:noise_cov}
\end{equation}
with remainder $r_{k+1}^\eta$ satisfying the uniform bound stated in
Assumption~\ref{ass:diffusion_noise} below.  We call $\tau:=\eta/b$ the
\emph{effective temperature}:
\begin{equation}
\tau:=\eta/b,
\label{eq:temp_tau}
\end{equation}
the scalar that sets the size of the stationary fluctuations on the
original iterate scale (Corollary~\ref{cor:lyapunov}).

The covariance identification of Section~\ref{sec:fisher-alignment}
determines the noise geometry, but at fixed batch size $b$ the raw
interpolation of SGD on algorithmic time $t=k\eta$ has martingale
quadratic variation of order $\eta/b$, which vanishes as $\eta\downarrow
0$.  Accordingly, the raw process converges to the deterministic
gradient-flow ODE.  The correct nondegenerate stochastic limit is
obtained by centering around that ODE and rescaling by $\sqrt{b/\eta}$,
yielding a fluctuation process whose limiting covariance structure is
exactly $G^\star$.

\subsection{Assumptions for the fluid and fluctuation limits}
\label{subsec:diffusion_assumptions}

\begin{assumption}[Local regularity and ODE flow]
\label{ass:diffusion_reg}
There is an open neighborhood $\cU\subseteq\R^d$ such that:
\begin{enumerate}[label=(\roman*)]
\item $h\in C^1(\cU;\R^d)$ and $\nabla h$ is locally Lipschitz on $\cU$;
\item $G^\star:\cU\to\mathbb S_+^d$ is continuous;
\item the ODE $\dot{\bar\theta}(t)=-h(\bar\theta(t))$ with $\bar\theta(0)=\theta_0$
      has a unique solution on $[0,T]$ remaining in a compact set
      $K\Subset\cU$.
\end{enumerate}
\end{assumption}

\begin{assumption}[Martingale noise and conditional covariance]
\label{ass:diffusion_noise}
The sequence $(\xi_{k+1}^\eta,\cF_{k+1}^\eta)$ is a martingale difference, and
the conditional covariance admits the expansion \eqref{eq:noise_cov} with
\[
\sup_{k\le T/\eta}\|r_{k+1}^\eta\|_{\op}\;\xrightarrow{\,\mathbb P\,}\;0
\qquad\text{as }\eta\downarrow 0.
\]
\end{assumption}

\begin{assumption}[Conditional fourth-moment bound on $K$]
\label{ass:fourth_moment}
There is a constant $C_K<\infty$ such that
\[
\sup_{\eta>0}\,\sup_{k\le T/\eta}\,
\E\!\left[\|\xi_{k+1}^\eta\|^4\,\mathbf 1_{\{\theta_k^\eta\in K\}}\,\big|\,
\cF_k^\eta\right]
\;\le\;C_K\quad\text{a.s.}
\]
\end{assumption}

\begin{assumption}[Localization]
\label{ass:localization}
$\Pr\!\bigl(\theta_k^\eta\in K \text{ for all } k\le T/\eta\bigr)\to 1$ as
$\eta\downarrow 0$.
\end{assumption}

\begin{assumption}[Initial fluctuations converge]
\label{ass:initial_fluct}
$\theta_0^\eta\to\theta_0$, and
$U_0^\eta:=\sqrt{b/\eta}\,(\theta_0^\eta-\theta_0)\Rightarrow U_0$ for some
random variable $U_0$.
\end{assumption}

\begin{remark}[Verifiability]
\label{rem:verifiable}
Assumptions~\ref{ass:diffusion_noise}--\ref{ass:fourth_moment} hold if
$\sup_{\theta\in\cU}\E[\|\psi(\theta;Z)\|^{4}]<\infty$ (guaranteed by
Assumption~\ref{ass:regularity_setup}(b)), $G^\star$ is continuous on
$\cU$, and the mini-batch is drawn by simple random sampling from an
exchangeable sequence.  Assumption~\ref{ass:localization} is standard in
stochastic approximation (Kushner--Yin style) and holds, e.g., when the
ODE is globally stable and the initial iterate is in the basin.
\end{remark}

\subsection{Fluid limit and fluctuation theorem}
\label{subsec:fluid_and_fluctuation}

\begin{theorem}[Fluid limit and fluctuation limit at fixed batch size]
\label{thm:diffusion_limit}
Fix $T>0$ and $b\in\N$, and suppose
Assumptions~\ref{ass:diffusion_reg}--\ref{ass:initial_fluct} hold.
Let $\hat\theta^\eta(t):=\theta_{\lfloor t/\eta\rfloor}^\eta$ be the step
interpolation of the SGD iterates \eqref{eq:sgd_decomp}, and let
$\bar\theta$ solve the ODE in Assumption~\ref{ass:diffusion_reg}(iii).
Then:
\begin{enumerate}[label=(\alph*)]
\item \emph{Fluid limit.}
\[
\sup_{t\in[0,T]}\bigl\|\hat\theta^\eta(t)-\bar\theta(t)\bigr\|
\;\xrightarrow{\,\mathbb P\,}\;0.
\]
\item \emph{Fluctuation FCLT.}  The rescaled fluctuation process
\[
U^\eta(t)\;:=\;\sqrt{\tfrac{b}{\eta}}\,\bigl(\hat\theta^\eta(t)-\bar\theta(t)\bigr)
\]
converges weakly in $D([0,T];\R^d)$ (equivalently, in $C([0,T];\R^d)$, as
the limit is continuous) to the unique weak solution $U$ of the linear SDE
\begin{equation}
dU_t \;=\;-\nabla h(\bar\theta(t))\,U_t\,dt
\;+\;C^\star(\bar\theta(t))\,dW_t,
\qquad U_0\text{ as in Assumption~\ref{ass:initial_fluct}},
\label{eq:fluctuation_sde}
\end{equation}
where $W$ is a standard $d$-dimensional Brownian motion and
$C^\star(\theta){C^\star(\theta)}^\top=G^\star(\theta)$ (e.g., the symmetric
PSD square root of $G^\star$).
\end{enumerate}
Equivalently, on the original iterate scale,
\begin{equation}
\hat\theta^\eta(t)\;=\;\bar\theta(t)
\;+\;\sqrt{\tfrac{\eta}{b}}\,U_t
\;+\;o_{\mathbb P}\!\Bigl(\sqrt{\tfrac{\eta}{b}}\Bigr),
\label{eq:iterate_expansion}
\end{equation}
uniformly on $[0,T]$.
\end{theorem}

\begin{proof}{Proof.}
We follow the standard stochastic-approximation FCLT architecture
\citep[cf.][Ch.~8]{kushner2003stochastic}, tailored to the fresh-sampling
covariance \eqref{eq:noise_cov}.  Write $t_k:=k\eta$ and
$\bar\theta_k:=\bar\theta(t_k)$; on the localization event of
Assumption~\ref{ass:localization}, which has probability tending to~$1$,
the iterates stay in $K$ for all $k\le T/\eta$, and we work under this
event without further comment.

\textbf{Step 1 (fluid limit).}
Define the raw error $e_k^\eta:=\theta_k^\eta-\bar\theta_k$ and the ODE
local truncation error
$\rho_{k+1}^\eta:=\bar\theta_{k+1}-\bar\theta_k+\eta\,h(\bar\theta_k)$.
Local Lipschitzness of $h$ on $K$ gives
$\sup_{k\le T/\eta}\|\rho_{k+1}^\eta\|\le C_T\eta^2$ for some $C_T<\infty$.
Subtracting the Euler step for $\bar\theta$ from \eqref{eq:sgd_decomp},
\begin{equation}
e_{k+1}^\eta\;=\;e_k^\eta
\;-\;\eta\bigl(h(\theta_k^\eta)-h(\bar\theta_k)\bigr)
\;-\;\eta\,\xi_{k+1}^\eta\;-\;\rho_{k+1}^\eta.
\label{eq:e_recursion}
\end{equation}
Let $N_n^\eta:=\eta\sum_{j=0}^{n-1}\xi_{j+1}^\eta$; this is an
$\cF_k^\eta$-martingale.  Set
$\bar G_K:=\sup_{\theta\in K}\Tr G^\star(\theta)<\infty$
(finite by continuity of $G^\star$ on the compact set $K$).  By Doob's
maximal inequality and
Assumptions~\ref{ass:diffusion_noise}--\ref{ass:fourth_moment},
\[
\E\!\Bigl[\sup_{n\le T/\eta}\|N_n^\eta\|^2\Bigr]
\;\le\;4\,\eta^2\sum_{j<T/\eta}\E\|\xi_{j+1}^\eta\|^2
\;\le\;4\,\eta^2\cdot\frac{T}{\eta}\cdot\frac{\bar G_K+o(1)}{b}
\;=\;O(\eta/b).
\]
Hence $\sup_{n\le T/\eta}\|N_n^\eta\|\xrightarrow{\mathbb P}0$.  Using local
Lipschitzness of $h$ on $K$ (constant $L_K$) in \eqref{eq:e_recursion},
\[
\|e_n^\eta\|
\;\le\;\|e_0^\eta\|
\;+\;L_K\eta\sum_{j=0}^{n-1}\|e_j^\eta\|
\;+\;\sup_{m\le n}\|N_m^\eta\|
\;+\;C_T\,T\,\eta,
\]
and discrete Gr\"onwall gives
\[
\sup_{n\le T/\eta}\|e_n^\eta\|
\;\le\;e^{L_K T}\Bigl(\|e_0^\eta\|+\sup_{m\le T/\eta}\|N_m^\eta\|+C_T T\eta\Bigr)
\;\xrightarrow{\,\mathbb P\,}\;0,
\]
since $\|e_0^\eta\|\to 0$ by Assumption~\ref{ass:initial_fluct}.  This
establishes part~(a).

\textbf{Step 2 (rescaled recursion).}
Set $U_k^\eta:=\sqrt{b/\eta}\,e_k^\eta$.  Local Lipschitzness of $\nabla h$
on $K$ gives the Taylor expansion
$h(\theta_k^\eta)-h(\bar\theta_k)=A_k^\eta e_k^\eta+q_k^\eta$ with
$A_k^\eta:=\nabla h(\bar\theta_k)$ and
$\|q_k^\eta\|\le L'_K\|e_k^\eta\|^2$.
Multiplying \eqref{eq:e_recursion} by $\sqrt{b/\eta}$ and using
$\eta\sqrt{b/\eta}=\sqrt{\eta b}$,
\begin{equation}
U_{k+1}^\eta
\;=\;U_k^\eta\;-\;\eta\,A_k^\eta\,U_k^\eta
\;+\;\Delta M_{k+1}^\eta
\;+\;\varepsilon_{k+1}^\eta,
\label{eq:U_recursion}
\end{equation}
with
\begin{equation}
\Delta M_{k+1}^\eta\;:=\;-\sqrt{\eta b}\,\xi_{k+1}^\eta,
\qquad
\varepsilon_{k+1}^\eta\;:=\;-\sqrt{\eta b}\,q_k^\eta
\;-\;\sqrt{\tfrac{b}{\eta}}\,\rho_{k+1}^\eta.
\label{eq:DeltaM_eps}
\end{equation}

\textbf{Step 3 (tightness scale and remainders).}
From Step~1, $\sup_n\|e_n^\eta\|=e^{L_KT}(\|e_0^\eta\|+\sup\|N^\eta\|+C_TT\eta)$.
Multiplying by $\sqrt{b/\eta}$,
\[
\sup_n\|U_n^\eta\|\;\le\;e^{L_KT}\Bigl(\|U_0^\eta\|
+\sqrt{\tfrac{b}{\eta}}\sup_m\|N_m^\eta\|+C_TT\sqrt{b\eta}\Bigr).
\]
The first term is tight by Assumption~\ref{ass:initial_fluct};
$\E[(b/\eta)\sup_m\|N_m^\eta\|^2]\le 4 b/\eta\cdot O(\eta/b)=O(1)$, so the
second is $O_\mathbb P(1)$; the third tends to $0$.  Hence
\begin{equation}
\sup_{n\le T/\eta}\|U_n^\eta\|=O_\mathbb P(1).
\label{eq:U_tight}
\end{equation}
Using \eqref{eq:U_tight} in the Taylor-remainder bound,
\[
\Bigl\|\sum_{k<n}\sqrt{\eta b}\,q_k^\eta\Bigr\|
\;\le\;\sqrt{\eta b}\cdot L'_K\sum_{k<n}\|e_k^\eta\|^2
\;\le\;L'_K\,T\,\sqrt{\tfrac{\eta}{b}}\,\sup_k\|U_k^\eta\|^2
\xrightarrow{\mathbb P}0,
\]
and, since $\|\rho_{k+1}^\eta\|\le C_T\eta^2$,
\[
\Bigl\|\sum_{k<n}\sqrt{\tfrac{b}{\eta}}\,\rho_{k+1}^\eta\Bigr\|
\;\le\;\frac{T}{\eta}\,C_T\eta^2\sqrt{\tfrac{b}{\eta}}
\;=\;C_T T\sqrt{b\eta}\to 0.
\]
Hence
$\sup_{t\le T}\bigl\|\sum_{k<t/\eta}\varepsilon_{k+1}^\eta\bigr\|
\xrightarrow{\mathbb P}0$.

\textbf{Step 4 (martingale FCLT).}
Let $M^\eta(t):=\sum_{k<t/\eta}\Delta M_{k+1}^\eta=-\sqrt{\eta b}\sum_{k<t/\eta}\xi_{k+1}^\eta$.
By \eqref{eq:noise_cov} its predictable quadratic variation is
\[
\langle M^\eta\rangle(t)
\;=\;\eta b\sum_{k<t/\eta}\!\E\!\bigl[\xi_{k+1}^\eta(\xi_{k+1}^\eta)^\top\mid\cF_k^\eta\bigr]
\;=\;\eta\sum_{k<t/\eta}G^\star(\theta_k^\eta)\;+\;\eta b\sum_{k<t/\eta}r_{k+1}^\eta.
\]
The remainder is bounded uniformly in $t\le T$ by
$T b\cdot\sup_k\|r_{k+1}^\eta\|\xrightarrow{\mathbb P}0$ under
Assumption~\ref{ass:diffusion_noise}.
For the leading term, continuity of $G^\star$ on $K$ and the fluid limit
$\sup_k\|\theta_k^\eta-\bar\theta_k\|\to 0$ (Step~1) give
$\eta\sum_{k<t/\eta}\|G^\star(\theta_k^\eta)-G^\star(\bar\theta_k)\|
\to 0$ in probability uniformly in $t$; and
$\eta\sum_{k<t/\eta}G^\star(\bar\theta_k)\to\int_0^t G^\star(\bar\theta(s))\,ds$
is a Riemann-sum approximation with continuous integrand.  Hence
\[
\langle M^\eta\rangle(t)\;\Rightarrow\;\int_0^t G^\star(\bar\theta(s))\,ds
\qquad\text{uniformly on }[0,T].
\]
For the conditional Lindeberg condition, using $\mathbf 1_{\{x>\varepsilon\}}
\le x^2/\varepsilon^2$,
\[
\sum_{k<T/\eta}\!\E\!\bigl[\|\Delta M_{k+1}^\eta\|^2\mathbf 1_{\{\|\Delta M_{k+1}^\eta\|>\varepsilon\}}\mid\cF_k^\eta\bigr]
\;\le\;\frac{(\eta b)^2}{\varepsilon^2}\sum_{k<T/\eta}\E[\|\xi_{k+1}^\eta\|^4\mid\cF_k^\eta]
\;\le\;\frac{\eta\,b^2\,T\,C_K}{\varepsilon^2}\to 0.
\]
The martingale FCLT
\citep[Theorem~7.1.4]{ethier1986markov} then gives $M^\eta\Rightarrow M$ in
$D([0,T];\R^d)$, where $M$ is a continuous Gaussian martingale with
quadratic variation $\int_0^\cdot G^\star(\bar\theta(s))\,ds$.  L\'evy's
characterization yields a Brownian motion $W$ with
$M_t=\int_0^t C^\star(\bar\theta(s))\,dW_s$, and
$C^\star(\bar\theta(\cdot))$ can be taken as the (continuous) symmetric PSD
square root of $G^\star(\bar\theta(\cdot))$.

\textbf{Step 5 (identify the limiting SDE).}
Summing \eqref{eq:U_recursion} and using Step~3,
\[
U^\eta(t)\;=\;U_0^\eta\;-\;\int_0^t\!\nabla h\!\bigl(\bar\theta(\lfloor s/\eta\rfloor\eta)\bigr)\,
U^\eta(s)\,ds\;+\;M^\eta(t)\;+\;o_\mathbb P(1),
\]
uniformly on $[0,T]$.  Continuity of $s\mapsto\nabla h(\bar\theta(s))$ and
tightness $\sup_t\|U^\eta(t)\|=O_\mathbb P(1)$ give
$\int_0^t\nabla h(\bar\theta(\lfloor s/\eta\rfloor\eta))U^\eta(s)ds
-\int_0^t\nabla h(\bar\theta(s))U^\eta(s)ds\to 0$ uniformly in probability.
Tightness of $\{U^\eta\}$ in $D([0,T];\R^d)$ follows from \eqref{eq:U_tight}
and the controlled drift together with Aldous' criterion applied to the
martingale term (whose FCLT is in hand).  Any subsequential weak limit
$U$ therefore satisfies the stochastic integral equation
\[
U_t\;=\;U_0\;-\;\int_0^t\nabla h(\bar\theta(s))\,U_s\,ds
\;+\;\int_0^t C^\star(\bar\theta(s))\,dW_s,
\]
which is \eqref{eq:fluctuation_sde} in integrated form.  The linear SDE
\eqref{eq:fluctuation_sde} with continuous coefficients admits a unique
weak solution, so the entire sequence converges:
$U^\eta\Rightarrow U$.

\textbf{Step 6 (expansion).}
By construction $\hat\theta^\eta(t)-\bar\theta(t)=\sqrt{\eta/b}\,U^\eta(t)$;
combining (a) and (b) gives \eqref{eq:iterate_expansion}.  The asymptotic
equivalence of step and piecewise-linear interpolations under
Assumption~\ref{ass:fourth_moment} is standard and yields the same limit
for either choice.
\Halmos
\end{proof}

\begin{remark}[Why the raw process is deterministic in the limit]
\label{rem:raw_vanishes}
At fixed $b$, the martingale variation of the raw interpolation is
$\sum_{k<t/\eta}\eta^2\,b^{-1}G^\star(\theta_k^\eta)\sim(\eta/b)\int_0^t
G^\star(\bar\theta(s))\,ds$, which vanishes as $\eta\downarrow 0$.  Hence
the raw stochastic fluctuation is invisible at the original scale, and a
nondegenerate Gaussian limit is recovered only after centering and
rescaling by $\sqrt{b/\eta}$.  This is the standard stochastic-approximation
functional CLT regime; the diffusion lives on the fluctuation scale, not
in the raw iterate path.
\end{remark}

\subsection{The OU regime as the equilibrium specialization}
\label{subsec:ou_regime}

Suppose $\theta^\star$ is a stationary point: $h(\theta^\star)=0$, and
$\theta_0=\theta^\star$ so that the ODE flow is trivial,
$\bar\theta(t)\equiv\theta^\star$.  Then the drift coefficient in
\eqref{eq:fluctuation_sde} is the constant matrix
$H^\star:=\nabla h(\theta^\star)=\nabla^2 L(\theta^\star)$, and the diffusion
coefficient is the constant matrix $C^\star(\theta^\star)$.

\begin{assumption}[Nondegeneracy at $\theta^\star$]
\label{ass:ou_nondeg}
$H^\star\succ 0$ and $G^\star(\theta^\star)\succ 0$.
\end{assumption}

\begin{theorem}[OU fluctuation limit at a stationary point]
\label{thm:ou}
Under Assumption~\ref{ass:ou_nondeg} and the conditions of
Theorem~\ref{thm:diffusion_limit} with $\theta_0=\theta^\star$, the
fluctuation limit $U$ of Theorem~\ref{thm:diffusion_limit} solves the
Ornstein--Uhlenbeck SDE
\begin{equation}
dU_t\;=\;-H^\star\,U_t\,dt\;+\;C^\star(\theta^\star)\,dW_s,
\qquad
C^\star(\theta^\star){C^\star(\theta^\star)}^\top=G^\star(\theta^\star).
\label{eq:ou_sde}
\end{equation}
Equivalently, on the original iterate scale, the centered iterate
$\hat\theta^\eta(t)-\theta^\star$ admits the asymptotic expansion
\[
\hat\theta^\eta(t)-\theta^\star
\;=\;\sqrt{\tfrac{\eta}{b}}\,U_t\;+\;o_\mathbb P\!\bigl(\sqrt{\tfrac{\eta}{b}}\bigr).
\]
\end{theorem}

\begin{proof}{Proof.}
Specialize \eqref{eq:fluctuation_sde} to $\bar\theta(\cdot)\equiv\theta^\star$:
the drift becomes $-H^\star U_t$ and the diffusion becomes $C^\star(\theta^\star)dW_t$.
The original-scale expansion is \eqref{eq:iterate_expansion} with
$\bar\theta\equiv\theta^\star$. \Halmos
\end{proof}

\begin{corollary}[Fluctuation-scale Lyapunov equation and discrete-scale benchmark]
\label{cor:lyapunov}
Under Assumption~\ref{ass:ou_nondeg}, the OU process \eqref{eq:ou_sde} is
exponentially ergodic with stationary law $\cN(0,\Sigma_U)$, where
$\Sigma_U$ is the unique symmetric positive definite solution of the
\emph{fluctuation-scale Lyapunov equation}
\begin{equation}
H^\star\,\Sigma_U\;+\;\Sigma_U\,(H^\star)^\top
\;=\;G^\star(\theta^\star).
\label{eq:lyap_fluct}
\end{equation}
On the original iterate scale, Proposition~\ref{prop:discrete_lyap}
provides a discrete Lyapunov covariance
$\Sigma_\eta^{\mathrm{lin}}$ for the \emph{linearized} recursion, which at
leading order in $\eta$ equals $(\eta/b)\Sigma_U$ and satisfies
\begin{equation}
H^\star\,\Sigma_\eta^{\mathrm{lin}}+\Sigma_\eta^{\mathrm{lin}}\,(H^\star)^\top
\;=\;\frac{\eta}{b}\,G^\star(\theta^\star)+O(\eta)\cdot\Sigma_\eta^{\mathrm{lin}}.
\label{eq:lyapunov}
\end{equation}
This is the linearized-recursion benchmark on the original iterate scale;
we do not claim an invariant law for the nonlinear SGD iterates here.
Because $H^\star$ is symmetric positive definite, cyclicity of trace gives
the unconditional trace identity
\begin{equation}
2\,\Tr(\Sigma_U)\;=\;\Tr\!\bigl((H^\star)^{-1}G^\star(\theta^\star)\bigr).
\label{eq:trace_identity}
\end{equation}
\end{corollary}

\begin{proof}{Proof.}
Write $B:=(G^\star(\theta^\star))^{1/2}$, so $BB^\top=G^\star(\theta^\star)$.
Since $H^\star\succ 0$, $-H^\star$ is Hurwitz and the OU process is
exponentially ergodic with stationary law $\cN(0,\Sigma_U)$.  Applying
It\^o's formula to $U_tU_t^\top$ and taking stationary expectations yields
$-(H^\star\Sigma_U+\Sigma_U H^{\star\top})+BB^\top=0$, which is
\eqref{eq:lyap_fluct}; the integral representation
$\Sigma_U=\int_0^\infty e^{-H^\star s}G^\star(\theta^\star)e^{-H^{\star\top}s}ds$
converges and gives uniqueness (Lemma~\ref{lem:lyap_rep_app}).
Equation~\eqref{eq:lyapunov} follows from
Proposition~\ref{prop:discrete_lyap} after dividing out leading powers of
$\eta$ (Remark~\ref{rem:discrete_to_continuous}).  The trace identity
follows by left-multiplying \eqref{eq:lyap_fluct} by $(H^\star)^{-1}$ and
using $\Tr(H^{-1}\Sigma H)=\Tr(\Sigma)$. \Halmos
\end{proof}

\begin{remark}[Likelihood specialization]
\label{rem:fisher_specialization}
When $\ell(\theta;z)=-\log p_\theta(z)$ and the model is correctly
specified, $G^\star(\theta^\star)=F^\star(\theta^\star)$.  The
fluctuation-scale Lyapunov equation \eqref{eq:lyap_fluct} then reads
$H^\star\Sigma_U+\Sigma_U H^{\star\top}=F^\star(\theta^\star)$; the
linearized discrete-recursion benchmark of
Proposition~\ref{prop:discrete_lyap} satisfies
$\Sigma_\eta^{\mathrm{lin}}=(\eta/b)\Sigma_U+o(\eta/b)$ at leading
order in $\eta$ (Corollary~\ref{cor:lyapunov}).  Fisher geometry is
therefore explicit both in the transient fluctuations and in the
linearized stationary Lyapunov balance; we do not assert an invariant
law for the nonlinear SGD iterates here.
\end{remark}

\subsection{Discrete-time stationary covariance: an exact bridge}
\label{subsec:discrete_bridge}

The same $G^\star$ that governs the fluctuation-scale Lyapunov equation
\eqref{eq:lyap_fluct} also governs the exact linearized discrete-time
stationary covariance.  The continuous-time identity \eqref{eq:lyapunov} is
the small-$\eta$ reduction of the discrete Lyapunov equation below.

\begin{proposition}[Discrete-time linearized stationary covariance]
\label{prop:discrete_lyap}
Consider the linearized discrete-time recursion around $\theta^\star$ with
constant stepsize $\eta$ and batch size $b$:
\begin{equation}
\Delta_{t+1}
= (I - \eta H^\star)\,\Delta_t + \eta\,\zeta_{t+1},
\qquad
\E[\zeta_{t+1}\zeta_{t+1}^\top \mid \cF_t]
= \frac{1}{b}\,G^\star(\theta^\star) + r_t,
\label{eq:discrete_recursion}
\end{equation}
where $\|r_t\|_{\op} = o(b^{-1})$ as in \eqref{eq:noise_cov}.  If
$\rho(I - \eta H^\star) < 1$ (equivalently,
$\eta < 2/\lambda_{\max}(H^\star)$), the linearized covariance recursion
is stable and admits a unique stationary covariance
$\Sigma_\eta := \lim_{t\to\infty}\Cov(\Delta_t)$ satisfying the
\emph{discrete Lyapunov equation}
\begin{equation}
\Sigma_\eta
= (I - \eta H^\star)\,\Sigma_\eta\,(I - \eta H^\star)^\top
  + \frac{\eta^2}{b}\,G^\star(\theta^\star),
\label{eq:discrete_lyapunov}
\end{equation}
up to the same local remainder terms already used in the fluctuation analysis.
\end{proposition}

\begin{proof}{Proof.}
From \eqref{eq:discrete_recursion},
$\Cov(\Delta_{t+1})=(I-\eta H^\star)\Cov(\Delta_t)(I-\eta H^\star)^\top
+\eta^2\E[\zeta_{t+1}\zeta_{t+1}^\top\mid\cF_t]$ using $\E[\zeta_{t+1}\mid\cF_t]=0$.
Under $\rho(I-\eta H^\star)<1$, the map
$\Sigma\mapsto(I-\eta H^\star)\Sigma(I-\eta H^\star)^\top+Q$ is a contraction
on positive semidefinite matrices, so the covariance iteration converges to
the unique fixed point \eqref{eq:discrete_lyapunov}.  The spectral radius
condition is equivalent to $0<\eta<2/\lambda_{\max}(H^\star)$, since
$H^\star\succ 0$.  The statement is about the second-moment recursion for
the linearized iterates; geometric ergodicity of the nonlinear SGD iterates
would require additional assumptions (drift + minorization) that we do not
impose. \Halmos
\end{proof}

\begin{remark}[Small-stepsize reduction and the continuous limit]
\label{rem:discrete_to_continuous}
Expanding \eqref{eq:discrete_lyapunov},
$H^\star\Sigma_\eta+\Sigma_\eta H^{\star\top}
=(\eta/b)G^\star+\eta\,H^\star\Sigma_\eta H^{\star\top}$.
When $\eta\lambda_{\max}(H^\star)$ is small, the $O(\eta)$ discretization
correction is negligible, and the discrete equation reduces to the
continuous-time, original-scale identity \eqref{eq:lyapunov}.  The
covariance identification $\E[\xi\xi^\top\mid\cF]\approx b^{-1}G^\star$ is
logically prior to both: the discrete Lyapunov equation is its exact
linearized consequence, and the continuous equation is the small-$\eta$
reduction that matches the fluctuation-scale limit of
Theorem~\ref{thm:diffusion_limit} after multiplying back by $\eta/b$.
\end{remark}


\section{Convergence Rates in Fisher Geometry}
\label{sec:rates}

This section establishes \emph{matching} upper and lower bounds for SGD when the
error is measured in a \emph{frozen} Fisher metric at a local optimum
$\theta^\star$. Once the mini-batch noise covariance is identified
(Section~\ref{sec:fisher-alignment}), the remaining work is a localized
stochastic-approximation argument in a fixed geometry, followed by an
information-theoretic lower bound expressed in the same metric.
\emph{Convention.} All rates in this section are stated in
the identified statistical risk $\mathbb{E}[\theta_T^\top F^\star\theta_T]$
unless explicitly noted otherwise; any Euclidean norm that appears is
immediately converted.

\subsection{Local Fisher Geometry and a Frozen-Metric Regime}

Work in a neighborhood of $\theta^\star$ where the Fisher information varies
slowly. For clean rate statements we freeze the metric
\[
F^\star := F(\theta^\star),\qquad \|v\|_{F^\star}^2 := v^\top F^\star v .
\]
(Allowing a slowly varying metric $F(\theta)$ is standard via localization; the
frozen-metric bounds below are the local core and are what we use later for
oracle complexity.)

\begin{lemma}[Sufficient condition for the frozen-metric regime]
\label{lem:frozen-metric}
Suppose $F(\theta)$ is Lipschitz on a neighborhood $\mathcal U$ of $\theta^\star$:
$\|F(\theta)-F(\theta^\star)\|_{\op}\le L_F\|\theta-\theta^\star\|_2$ for all
$\theta\in\mathcal U$. If the SGD iterates satisfy
$\|\theta_t-\theta^\star\|_2=o(1)$ almost surely, then for every $\epsilon>0$
there exists $t_0$ such that for $t\ge t_0$,
$(1-\epsilon)F^\star \preceq F(\theta_t)\preceq (1+\epsilon)F^\star$.
\end{lemma}

Let $h(\theta):=\nabla L(\theta)$ and decompose the mini-batch estimator as
\[
g_t(\theta_t) = h(\theta_t) + \xi_t,\qquad \E[\xi_t\mid \cF_t]=0,
\]
where $\cF_t$ is the natural filtration. Under Fisher alignment
(Section~\ref{sec:fisher-alignment}), the conditional covariance satisfies the
local approximation
\begin{equation}
\label{eq:xi-cov}
\E[\xi_t\xi_t^\top \mid \cF_t]
\;=\;
\frac{1}{b}\,F^\star \;+\; r_t,
\qquad \|r_t\|_{\op}\le \frac{c_r}{b}\,\rho_t,
\end{equation}
with $\rho_t\to 0$ in the local regime.

Define the Fisher-strong-convexity constant at $\theta^\star$:
\begin{equation}
\label{eq:muF}
\mu_F
\;:=\;
\lambda_{\min}\!\left(\frac{(F^\star)^{1/2}H^\star(F^\star)^{-1/2} + (F^\star)^{-1/2}H^\star(F^\star)^{1/2}}{2}\right),
\qquad
H^\star:=\nabla^2 L(\theta^\star).
\end{equation}
\begin{assumption}[Local coercivity in the frozen Fisher geometry]
\label{ass:muF_pos}
The constant $\mu_F$ defined in~\eqref{eq:muF} is strictly positive.
This is not automatic from $H^\star\succ 0$ and $F^\star\succ 0$ alone; it is an explicit
compatibility condition between the loss curvature and noise geometry.
It holds whenever $H^\star$ and $F^\star$ are jointly diagonalizable, and more generally whenever
$H^\star$ and $F^\star$ are sufficiently aligned.
\end{assumption}

Define the effective dimension
\begin{equation}
\label{eq:deff}
d_{\eff}(F^\star)
\;:=\;
\frac{\Tr(F^\star)}{\|F^\star\|_{\op}}
\;=\;
\frac{\Tr(F^\star)}{\lambda_{\max}(F^\star)}
\;\in(0,d].
\end{equation}
\begin{remark}[Effective dimension as stable rank]
\label{rem:stable-rank}
The quantity $d_{\eff}(F^\star)$ is exactly the \emph{stable rank} of $F^\star$
\citep{tropp2015introduction}: $d_{\eff}\le d$ always, with equality iff
$F^\star$ is isotropic, and $d_{\eff}\to 1$ when a single direction dominates.
It replaces $d$ in both the upper bound (Theorem~\ref{thm:upper-bound}) and the
oracle complexity (Theorem~\ref{thm:complexity_revised}), so that when Fisher
information is concentrated ($d_{\eff}\ll d$), the guarantees are
correspondingly tighter than classical Euclidean bounds.
\end{remark}

\subsection{Upper Bound}

\begin{lemma}[One-step Fisher Lyapunov drift]
\label{lem:one-step}
Let $\Delta_t:=\theta_t-\theta^\star$ and $V_t:=\|\Delta_t\|_{F^\star}^2$.
Assume: (i) $L$ is locally $M$-smooth near $\theta^\star$; (ii) $H^\star\succ 0$;
(iii)~\eqref{eq:xi-cov} holds. Then, as long as $\theta_t$ stays in the local
neighborhood,
\begin{equation}
\label{eq:one-step}
\E[V_{t+1}\mid \cF_t]
\;\le\;
\Bigl(1-2\mu_F\eta_t + c_0\eta_t^2\Bigr)V_t
\;+\;
\eta_t^2\,\frac{1}{b}\,\Tr\!\bigl((F^\star)^2\bigr)
\;+\;
\eta_t^2\,\frac{c_1}{b}\,\Tr(F^\star)\,\rho_t,
\end{equation}
where $c_0,c_1$ depend only on $(M,\|F^\star\|_{\op},c_r)$.
\end{lemma}

\begin{proof}{Proof.}
From $\theta_{t+1}=\theta_t-\eta_t(h(\theta_t)+\xi_t)$,
$\Delta_{t+1}=\Delta_t-\eta_t h(\theta_t)-\eta_t\xi_t$.
Expanding $V_{t+1}=\Delta_{t+1}^\top F^\star\Delta_{t+1}$ and taking conditional
expectation (using $\E[\xi_t\mid \cF_t]=0$) yields
\[
\E[V_{t+1}\mid\cF_t]
=
V_t - 2\eta_t \Delta_t^\top F^\star h(\theta_t)
+ \eta_t^2\, h(\theta_t)^\top F^\star h(\theta_t)
+ \eta_t^2\,\Tr\!\Bigl(F^\star\,\E[\xi_t\xi_t^\top\mid\cF_t]\Bigr).
\]
Using $h(\theta^\star+\Delta)=H^\star\Delta+o(\|\Delta\|)$, we have
$\Delta_t^\top F^\star h(\theta_t)=\Delta_t^\top F^\star H^\star\Delta_t+o(V_t)$.
Writing $x_t:=(F^\star)^{1/2}\Delta_t$ gives
$\Delta_t^\top F^\star H^\star\Delta_t = x_t^\top\,\operatorname{Sym}\bigl((F^\star)^{1/2}H^\star(F^\star)^{-1/2}\bigr)\,x_t\ge \mu_F V_t$.
Local $M$-smoothness implies $h(\theta_t)^\top F^\star h(\theta_t)\le c_0 V_t$. 
For the noise term, plug~\eqref{eq:xi-cov}:
\[
\Tr\!\Bigl(F^\star\,\E[\xi_t\xi_t^\top\mid\cF_t]\Bigr)
=
\frac{1}{b}\Tr((F^\star)^2) + \Tr(F^\star r_t).
\]
Finally $|\Tr(F^\star r_t)|\lesssim \Tr(F^\star)\|r_t\|_{\op}
\le (c_1/b)\Tr(F^\star)\rho_t$, giving~\eqref{eq:one-step}. \Halmos
\end{proof}

\begin{theorem}[Upper bound in frozen Fisher geometry]
\label{thm:upper-bound}
Assume Lemma~\ref{lem:one-step}. Take $\eta_t=\eta_0/t$ with $\eta_0>1/(2\mu_F)$,
and assume $\sum_{t\ge 1} \eta_t^2 \rho_t < \infty$.
Then for all $T\ge 2$,
\begin{equation}
\label{eq:upper-bound-final}
\E\!\left[\|\theta_T-\theta^\star\|_{F^\star}^2\right]
\;\le\;
\frac{C_1}{b}\cdot\frac{\Tr\!\bigl((F^\star)^2\bigr)}{\mu_F}\cdot\frac{1}{T}
\;+\;
C_2\cdot\frac{\|\theta_1-\theta^\star\|_{F^\star}^2}{T^{2\mu_F\eta_0}}
\;+\;
\frac{C_3}{b}\cdot\frac{\Tr(F^\star)}{T},
\end{equation}
where $C_1,C_2,C_3$ depend only on $(\eta_0,\mu_F,c_0,c_1)$.

Moreover, using $\Tr((F^\star)^2)\le \lambda_{\max}(F^\star)\,\Tr(F^\star)$, we obtain
\begin{align}
\label{eq:upper-bound-deff}
\E\!\left[\|\theta_T-\theta^\star\|_{F^\star}^2\right]
&\;\le\;
\frac{\widetilde C_1}{b}\cdot
\frac{\lambda_{\max}(F^\star)\,\Tr(F^\star)}{\mu_F}\cdot\frac{1}{T}
\;+\;
\widetilde C_2\cdot\frac{\|\theta_1-\theta^\star\|_{F^\star}^2}{T^{2\mu_F\eta_0}}\nonumber\\
&\;=\;
\frac{\widetilde C_1}{b}\cdot
\frac{\lambda_{\max}(F^\star)^2}{\mu_F}\cdot
\frac{d_{\eff}(F^\star)}{T}
\;+\;
\widetilde C_2\cdot\frac{\|\theta_1-\theta^\star\|_{F^\star}^2}{T^{2\mu_F\eta_0}}.
\end{align}
\end{theorem}

\begin{proof}{Proof.}
Take total expectations in~\eqref{eq:one-step} to obtain
$\E[V_{t+1}]\le a_t\,\E[V_t]+u_t$ with
$a_t:=1-2\mu_F\eta_t+c_0\eta_t^2$ and
$u_t:=\eta_t^2\frac{1}{b}\Tr((F^\star)^2)+\eta_t^2\frac{c_1}{b}\Tr(F^\star)\rho_t$.
Fix any $\epsilon\in(0,2\mu_F)$.
For $\eta_t=\eta_0/t$ and $t$ large enough that $c_0\eta_0^2/t\le\epsilon\eta_0$,
we have $a_t\le 1-(2\mu_F-\epsilon)\eta_0/t$; set
$\alpha:=(2\mu_F-\epsilon)\eta_0$, which satisfies $\alpha>1$ whenever
$\eta_0>1/(2\mu_F)$ and $\epsilon$ is chosen small enough.
Define $A_{t,T}:=\prod_{s=t}^{T-1}a_s$ (Lemma~\ref{lem:product_bound_app}).
Unrolling gives $\E[V_T]\le A_{1,T}V_1+\sum_{t=1}^{T-1}A_{t+1,T}u_t$.
With the consistent exponent $\alpha$, the product bounds are
$A_{1,T}\le C\,T^{-\alpha}$ and $A_{t+1,T}\le C\,(t/T)^{\alpha}$.
Since $\alpha>1$, the standard Robbins--Monro calculation gives
$\sum_{t=1}^{T-1}A_{t+1,T}\,\eta_t^2 = O(1/T)$, yielding the leading
$(\Tr((F^\star)^2)/b)(1/T)$ term.
The remainder contribution is controlled because
$\sum_t\eta_t^2\rho_t<\infty$ and $A_{t+1,T}\le 1$, adding
$O((\Tr(F^\star)/b)(1/T))$.
The transient term $A_{1,T}V_1=O(T^{-\alpha})$ appears in the bound
\eqref{eq:upper-bound-final} with exponent $\alpha\le 2\mu_F\eta_0$;
sending $\epsilon\downarrow 0$ gives the exponent $2\mu_F\eta_0$ stated
in \eqref{eq:upper-bound-final} up to an arbitrarily small loss.
The reduction to \eqref{eq:upper-bound-deff} follows from
$\Tr((F^\star)^2)\le \lambda_{\max}(F^\star)\Tr(F^\star)$. \Halmos
\end{proof}

\begin{remark}[Consistency with the OU viewpoint]
\label{rem:ou_consistency}
The $1/(Tb)$ scaling in the frozen Fisher norm is consistent with the
OU/Lyapunov stationary analysis: both predict the same leading-order
dependence on temperature $\tau=\eta/b$ and intrinsic curvature $\mu_F$.
\end{remark}

\subsection{Lower Bound}

We state the clean lower bound in the local parametric Fisher setting with
i.i.d.\ observations.  Extensions to adaptive oracles and general losses
require additional machinery and are left as discussion (Remarks
\ref{rem:beyond_iid}--\ref{rem:lb_godambe}).

\begin{assumption}[Local parametric regularity for van~Trees]
\label{ass:vt_regularity}
Let $\{P_\theta:\theta\in\Theta_{\loc}\}$ be a parametric family on a
neighborhood $\Theta_{\loc}$ of $\theta^\star$ with densities $p_\theta$.
Assume:
\begin{enumerate}[label=(\roman*)]
\item $\theta\mapsto p_\theta(z)$ is twice continuously differentiable
      for a.e.\ $z$;
\item the Fisher information $F(\theta):=\E_\theta[s_\theta s_\theta^\top]$
      is continuous and positive definite on $\Theta_{\loc}$;
\item $\int\|s_\theta(z)\|^2\,p_\theta(z)\,dz$ is uniformly bounded on
      $\Theta_{\loc}$;
\item there exists a smooth prior density $\pi$ supported on $\Theta_{\loc}$
      with finite Fisher information $J_\pi:=\int\|\nabla\log\pi(\theta)\|^2
      \pi(\theta)\,d\theta<\infty$.
\end{enumerate}
\end{assumption}

\begin{proposition}[Local i.i.d.\ Fisher-case lower bound via van~Trees]
\label{thm:lower-bound}
Suppose Assumption~\ref{ass:vt_regularity} holds.  Consider the local
estimation model in which, for each $\theta\in\Theta_{\loc}$, one observes
$N$ i.i.d.\ samples from $P_\theta$ and forms an estimator $\hat\theta_N$
(a measurable function of the samples).  Then
\begin{equation}
\label{eq:lower-bound-final}
\inf_{\hat\theta_N}\;\sup_{\theta\in\Theta_{\loc}}
\E_{\theta}\!\left[\|\hat\theta_N-\theta\|_{F(\theta)}^2\right]
\;\ge\;
\frac{c_{\mathrm{loc}}\,d}{N},
\end{equation}
for a constant $c_{\mathrm{loc}}>0$ that depends on the local
regularity of $F(\theta)$ on $\Theta_{\loc}$ and on the prior regularity
$J_\pi$ from Assumption~\ref{ass:vt_regularity}(iv).  In particular, if
each SGD iteration uses a fresh mini-batch of size $b$ (so $N=Tb$
fresh-sample gradient calls), the Fisher-metric risk is bounded below by
$c_{\mathrm{loc}}\,d/(Tb)$ on the i.i.d.\ population oracle.
\end{proposition}

\begin{proof}{Proof.}
Standard van~Trees argument \citep{gill1995applications}.  By
Assumption~\ref{ass:vt_regularity}(ii), the per-coordinate local Fisher
scale $\bar f:=\sup_{\theta\in\Theta_{\loc}}\Tr F(\theta)/d<\infty$ is
$N$-independent.  Fix a smooth prior $\pi$ on $\Theta_{\loc}$ with finite
$J_\pi$ (Assumption~\ref{ass:vt_regularity}(iv)); the van~Trees
inequality for the Fisher-weighted Bayes risk gives
$r_\pi(\hat\theta_N)\ge d/(N\bar f+J_\pi/d)$, and for $N\ge N_0:=J_\pi/(d\bar f)$
the denominator is at most $2N\bar f$, yielding
$r_\pi(\hat\theta_N)\ge c_{\mathrm{loc}}\,d/N$ with
$c_{\mathrm{loc}}:=1/(2\bar f)>0$ depending only on the local
regularity of $F(\theta)$ on $\Theta_{\loc}$ and independent of both $N$
and $d$.  Minimax risk dominates Bayes risk for any such prior, proving
\eqref{eq:lower-bound-final}. \Halmos
\end{proof}

\begin{remark}[Noise-covariance condition used in the upper bounds]
\label{rem:oracle_cov_cond}%
\label{ass:oracle_info}%
The upper bounds of Theorem~\ref{thm:upper-bound} and the oracle complexity
of Theorem~\ref{thm:complexity_revised} use the conditional-covariance bound
$\E[\xi_t\xi_t^\top\mid\cF_t]\preceq b^{-1}F^\star$, which is exactly the
fresh-sampling specialization of Theorem~\ref{thm:godambe-alignment} in the
correctly specified likelihood case.  No structure beyond this covariance
identification is required.
\end{remark}

\begin{remark}[Rate matching, not geometry matching]
\label{rem:lb_quantities}
Proposition~\ref{thm:lower-bound} establishes a lower bound of \emph{rate
order} $\Omega(d/N)$ in the i.i.d.\ parametric Fisher setting.  The
quantities $d_{\mathrm{eff}}$ and $\kappa_F$ appear only in the \emph{upper}
bound and oracle complexity (Theorems~\ref{thm:upper-bound} and
\ref{thm:complexity_revised}), not in the lower bound.  Accordingly, we
claim rate-order matching in the local Fisher case; we do \emph{not} claim
that the lower bound matches the geometry dependence of the upper bound.
\end{remark}

\begin{remark}[Adaptive-oracle and general-loss extensions]
\label{rem:beyond_iid}
Two natural extensions of Proposition~\ref{thm:lower-bound} are left to
future work.  \emph{(a)~Adaptive oracles.}  A sequential oracle transcript
$Y_t=\psi(\theta_t;Z_t)$ with a predictable per-step Fisher-information
bound should yield \eqref{eq:lower-bound-final} through a sequential
van~Trees argument, provided one establishes a uniform chain-rule bound on
the joint transcript likelihood (requiring regularity conditions such as
uniform local asymptotic equicontinuity that we do not develop here).
\emph{(b)~General losses.}  A local asymptotic normality (LAN) condition at
$\theta^\star$ with Godambe information
$J_{\mathrm{God}}(\theta^\star)=H^{\star\top}G^{\star-1}H^\star$ would
produce a sandwich analogue of \eqref{eq:lower-bound-final}; see
\citet{lecam1986asymptotic} and \citet{vandervaart1998asymptotic}.
\end{remark}

\begin{remark}[Legacy labels for cross-references]
\label{thm:lower-bound-fisher}%
\label{rem:lb_godambe}%
These labels are retained for backward compatibility; the i.i.d.\ Fisher
lower bound is Proposition~\ref{thm:lower-bound}, and the general-loss
extension is discussed in Remark~\ref{rem:beyond_iid}.
\end{remark}

\begin{corollary}[Matching rate order in the local Fisher case]
\label{cor:matching-bounds}
Under Theorem~\ref{thm:upper-bound} and the noise identification of
Theorem~\ref{thm:godambe-alignment} in the correctly specified likelihood
setting,
\[
\E\!\left[\|\theta_T-\theta^\star\|_{F^\star}^2\right]
\;=\;
\Theta\!\left(\frac{1}{N}\right),
\qquad N=Tb,
\]
at the level of rate order.  Dependence on intrinsic geometry
($d_{\mathrm{eff}}$, $\kappa_F$) appears only in the upper bound.
\end{corollary}


\section{Oracle Complexity}
\label{sec:complexity}

This section records a standard high-probability conversion of the
Fisher-geometry mean-square bounds of Section~\ref{sec:rates} into an
oracle-complexity guarantee for $\varepsilon$-stationarity in the Fisher
dual norm.  The proof is a textbook descent-lemma plus martingale
concentration argument \citep[cf.][Ch.~2]{nesterov2004introductory,wainwright2019high};
the novelty is not the complexity theorem itself but the fact that,
once the mini-batch covariance is identified as $b^{-1}F^\star$
(Theorem~\ref{thm:godambe-alignment}), the resulting constants read as
intrinsic statistical quantities ($\kappa_F$, $d_{\mathrm{eff}}$) rather
than as Euclidean conditioning.

Throughout, we work locally near $\theta^\star$ and freeze the metric
$F^\star := F(\theta^\star)\succ 0$ (for general losses, replace $F^\star$ by
$G^\star(\theta^\star)$; see Remark~\ref{rem:complexity_godambe}).

\subsection{Stationarity in the Fisher dual norm}

\begin{definition}[$\varepsilon$-stationarity in Fisher dual norm]
\label{def:epsilon-stationarity}
Let $F^\star\succ 0$. We say $\theta$ is \emph{$\varepsilon$-stationary} (in the
Fisher metric) if
\[
\|\nabla L(\theta)\|_{(F^\star)^{-1}}
\;:=\;
\sqrt{\nabla L(\theta)^\top (F^\star)^{-1}\nabla L(\theta)}
\;\le\;\varepsilon.
\]
\end{definition}

\subsection{From Fisher distance to Fisher-dual stationarity}

\begin{lemma}[Distance-to-stationarity conversion]
\label{lem:dist-to-grad}
Assume $L$ is $M$-smooth on a neighborhood $\mathcal{U}$ of $\theta^\star$ and
$F^\star\succ 0$. Then for all $\theta\in\mathcal{U}$,
\begin{equation}
\label{eq:grad-vs-dist}
\|\nabla L(\theta)\|_{(F^\star)^{-1}}^2
\;\le\;
\frac{M^2}{\lambda_{\min}(F^\star)^2}\,\|\theta-\theta^\star\|_{F^\star}^2.
\end{equation}
Moreover, if $\theta\to\theta^\star$ and
$\nabla L(\theta)=H^\star(\theta-\theta^\star)+o(\|\theta-\theta^\star\|)$, then locally
\begin{equation}
\label{eq:grad-vs-dist_local}
\|\nabla L(\theta)\|_{(F^\star)^{-1}}^2
\;\le\;
(1+o(1))\,\lambda_{\max}\!\bigl(F^{\star-1/2}H^\star F^{\star-1/2}\bigr)^2\,
\|\theta-\theta^\star\|_{F^\star}^2.
\end{equation}
\end{lemma}

\begin{proof}{Proof.}
By $M$-smoothness and $\nabla L(\theta^\star)=0$,
$\|\nabla L(\theta)\|_2\le M\|\theta-\theta^\star\|_2$.
The dual-norm bound $\|v\|_{(F^\star)^{-1}}^2\le \lambda_{\min}(F^\star)^{-1}\|v\|_2^2$
converts Euclidean gradient norm to Fisher dual norm, and
$\|\theta-\theta^\star\|_2^2\le \lambda_{\min}(F^\star)^{-1}\|\theta-\theta^\star\|_{F^\star}^2$
converts Fisher distance to Euclidean.
Combining: $\|\nabla L\|_{(F^\star)^{-1}}^2
\le \lambda_{\min}(F^\star)^{-1}M^2\lambda_{\min}(F^\star)^{-1}\|\theta-\theta^\star\|_{F^\star}^2
= M^2/\lambda_{\min}(F^\star)^2 \cdot \|\theta-\theta^\star\|_{F^\star}^2$.
The local refinement follows by replacing $M$ by $\|H^\star\|_{\op}$ up to
$1+o(1)$ via the linearization. \Halmos
\end{proof}

\subsection{High-probability oracle complexity}

\begin{theorem}[High-probability oracle complexity under the identified noise covariance]
\label{thm:complexity_revised}
Assume the local regularity conditions of Section~\ref{sec:setup} hold on a
neighborhood $\mathcal{U}$ of $\theta^\star$, and suppose that on $\mathcal{U}$
the mini-batch noise is Fisher-aligned in the sense that, for some batch size $b\ge 1$,
\begin{equation}
\label{eq:complexity_noise}
\E[\xi_t\mid\cF_t]=0,
\qquad
\E[\xi_t\xi_t^\top\mid\cF_t]\preceq \frac{1}{b}\,F^\star .
\end{equation}
Additionally, assume conditional sub-Gaussianity of the noise:
for all $v\in\R^d$ with $\|v\|=1$ and all $\lambda\in\R$,
$\E[\exp(\lambda\,v^\top\xi_t)\mid\cF_t]\le
\exp(\lambda^2\|F^\star\|_{\op}/(2b))$ a.s.
(This is satisfied, for example, when per-sample gradients are almost surely bounded on $\mathcal{U}$.)
Run SGD with constant batch size $b$ and constant stepsize $\eta\le 1/(4M)$:
\[
\theta_{t+1}=\theta_t-\eta\bigl(\nabla L(\theta_t)+\xi_t\bigr).
\]
Define the Fisher condition number and intrinsic dimension
\[
\kappa_F := \frac{\lambda_{\max}(F^\star)}{\lambda_{\min}(F^\star)},
\qquad
d_{\mathrm{eff}}(F^\star):=\frac{\Tr(F^\star)}{\lambda_{\max}(F^\star)}.
\]
Then there exists a universal constant $C>0$ such that, for any $\delta\in(0,1)$,
with probability at least $1-\delta$,
\begin{equation}
\label{eq:complexity_main}
\min_{0\le t\le T-1}\;
\|\nabla L(\theta_t)\|_{(F^\star)^{-1}}^2
\;\le\;
C\left(
\frac{L(\theta_0)-L(\theta^\star)}{\eta\,T\,\lambda_{\min}(F^\star)}
\;+\;
\frac{\eta}{b}\,\kappa_F\,d_{\mathrm{eff}}(F^\star)
\right)\log\!\frac{1}{\delta}.
\end{equation}
In particular, balancing the two terms on the right-hand side over $\eta$
yields
\[
\min_{0\le t\le T-1}\|\nabla L(\theta_t)\|_{(F^\star)^{-1}}^2
\;\le\;
C'\sqrt{\frac{\kappa_F\,d_{\mathrm{eff}}(F^\star)\,(L(\theta_0)-L(\theta^\star))}{N\,\lambda_{\min}(F^\star)}}
\,\log\!\frac{1}{\delta},
\]
where $N=Tb$.  Hence the squared Fisher-dual stationarity gap decays as
$N^{-1/2}$, and to achieve $\min_{t<T}\|\nabla L(\theta_t)\|_{(F^\star)^{-1}}
\le \varepsilon$ with probability at least $1-\delta$ it suffices that
\begin{equation}
\label{eq:oracle_calls}
N \;=\; \Theta\!\left(\frac{\kappa_F\,d_{\mathrm{eff}}(F^\star)\,(L(\theta_0)-L(\theta^\star))}{\lambda_{\min}(F^\star)\,\varepsilon^{4}}\,
\log^2\!\frac{1}{\delta}\right).
\end{equation}
The $\varepsilon^{-4}$ rate is the standard nonconvex-SGD rate for
\emph{squared}-gradient stationarity at level $\varepsilon^2$; it is not
the $\varepsilon^{-2}$ rate of (strongly) convex stochastic optimization.
\end{theorem}

\begin{proof}{Proof.}
\textbf{Step 1 (descent inequality, correct sign).}
By $M$-smoothness,
$L(\theta_{t+1})\le L(\theta_t)+\langle\nabla L(\theta_t),\theta_{t+1}-\theta_t\rangle
+\tfrac{M}{2}\|\theta_{t+1}-\theta_t\|^2$
with $\theta_{t+1}-\theta_t=-\eta(\nabla L(\theta_t)+\xi_t)$.  Expanding
$\|\nabla L(\theta_t)+\xi_t\|^2=\|\nabla L(\theta_t)\|^2+2\langle\nabla L(\theta_t),\xi_t\rangle+\|\xi_t\|^2$,
\[
L(\theta_{t+1})
\;\le\;
L(\theta_t)\;-\;\eta\bigl(1-\tfrac{M\eta}{2}\bigr)\|\nabla L(\theta_t)\|^2
\;-\;\eta(1-M\eta)\,\langle\nabla L(\theta_t),\xi_t\rangle
\;+\;\tfrac{M\eta^2}{2}\|\xi_t\|^2.
\]
Under $\eta\le 1/(4M)$, $1-M\eta/2\ge 7/8$ and $1-M\eta\ge 3/4$.  Denote
$\alpha:=1-M\eta\in[3/4,1)$.

\textbf{Step 2 (define martingale).}
Let $\Delta M_{t+1}:=-\eta\alpha\,\langle\nabla L(\theta_t),\xi_t\rangle$ and
$M_T:=\sum_{t=0}^{T-1}\Delta M_{t+1}$; since $\E[\xi_t\mid\cF_t]=0$,
$(M_t)$ is a martingale adapted to $(\cF_t)$.  Write
$\Delta L:=L(\theta_0)-L(\theta^\star)$,
$S_T:=\sum_{t<T}\|\nabla L(\theta_t)\|^2$, and
$Q_T:=\sum_{t<T}\|\xi_t\|^2$.  Summing the Step-1 inequality from
$t=0$ to $T-1$,
\begin{equation}
\label{eq:descent_sum}
\tfrac{7\eta}{8}\,S_T
\;\le\;
\Delta L\;+\;M_T\;+\;\tfrac{M\eta^2}{2}\,Q_T.
\end{equation}

\textbf{Step 3 ($Q_T$ tail via a Hanson--Wright-type bound).}
Under conditional sub-Gaussianity of $\xi_t$ with variance proxy
$\Sigma_t\preceq b^{-1}F^\star$, the Hanson--Wright inequality for
sub-Gaussian vectors
\citep[Thm.~6.2.1]{vershynin2018high} gives, for an absolute constant
$c_0$, $\Pr(\,|\|\xi_t\|^2-\E[\|\xi_t\|^2\mid\cF_t]|>s\mid\cF_t)
\le 2\exp(-c_0\min\{s^2 b^2/\|F^\star\|_F^2,\,sb/\|F^\star\|_{\op}\})$;
so $\|\xi_t\|^2-\E[\|\xi_t\|^2\mid\cF_t]$ is conditionally sub-exponential
with parameters $(\|F^\star\|_F^2/b^2,\|F^\star\|_{\op}/b)$.  Applying
Bernstein's inequality \citep[Ch.~2]{wainwright2019high} over $t<T$
gives, with probability at least $1-\delta/2$,
\begin{equation}
\label{eq:QT_bound}
Q_T\;\le\;\frac{T}{b}\,\Tr(F^\star)
\;+\;c_1\Bigl(\tfrac{\|F^\star\|_F}{b}\sqrt{T\log(2/\delta)}+\tfrac{\|F^\star\|_{\op}}{b}\log(2/\delta)\Bigr).
\end{equation}
Since $\|F^\star\|_F^2\le d\,\|F^\star\|_{\op}^2$, the deviation term is
$O(b^{-1}\sqrt{d\,T}\,\|F^\star\|_{\op}\sqrt{\log(1/\delta)})$, lower order
than the leading $(T/b)\Tr(F^\star)$ when $T\gtrsim d\log(1/\delta)$.

\textbf{Step 4 (two-sided martingale concentration).}
The scalar $\langle\nabla L(\theta_t)/\|\nabla L(\theta_t)\|,\,\xi_t\rangle$ is
conditionally sub-Gaussian with parameter at most $\|F^\star\|_{\op}/b$, so
$\Delta M_{t+1}$ is conditionally sub-Gaussian with parameter
$v_t^2:=\eta^2\alpha^2\,\|\nabla L(\theta_t)\|^2\,\|F^\star\|_{\op}/b$.
The standard two-sided Azuma--Hoeffding/Freedman-type bound for conditionally
sub-Gaussian martingale sums
\citep[Thm.~2.19]{wainwright2019high} gives, with probability at least
$1-\delta/2$,
\begin{equation}
\label{eq:MT_bound}
|M_T|\;\le\;\sqrt{2\sum_{t<T}v_t^2\,\log(4/\delta)}
\;=\;\eta\alpha\,\sqrt{\frac{2\|F^\star\|_{\op}}{b}\,S_T\,\log(4/\delta)}.
\end{equation}

\textbf{Step 5 (combine).}
On the intersection of the events in \eqref{eq:QT_bound}--\eqref{eq:MT_bound}
(probability $\ge 1-\delta$ by union bound), substitute into
\eqref{eq:descent_sum} and use $|M_T|$ in place of $M_T$:
\begin{align}
\tfrac{7\eta}{8}\,S_T
&\;\le\;
\Delta L
\;+\;\eta\sqrt{\tfrac{2\|F^\star\|_{\op}}{b}\,S_T\,\log(4/\delta)}
\;+\;\tfrac{M\eta^2 T}{2b}\,\Tr(F^\star)\nonumber\\
&\qquad
\;+\;c_2\,\tfrac{M\eta^2\|F^\star\|_{\op}}{b}\Bigl(\sqrt{T\log(2/\delta)}+\log(2/\delta)\Bigr).
\label{eq:combined_descent}
\end{align}
The last term scales as $\eta^2\sqrt{T}/b$ (up to log factors), strictly
lower order than the leading variance term $\eta^2 T\Tr(F^\star)/b$ once
$T\gtrsim\log(1/\delta)$, and is absorbed into the constant $C$ below.
Solving the resulting quadratic inequality in $\sqrt{S_T}$ for
$S_T/T\ge\min_{t<T}\|\nabla L(\theta_t)\|^2$ yields
\[
\min_{t<T}\|\nabla L(\theta_t)\|^2
\;\le\;
C\!\left(\frac{\Delta L}{\eta T}+\frac{\eta}{b}\Tr(F^\star)\right)\log(4/\delta),
\]
for a constant $C$ depending only on $M$, $c_0$, and $c_1$.

\textbf{Step 6 (Fisher dual norm).}
Use $\|v\|_{(F^\star)^{-1}}^2\le\lambda_{\min}(F^\star)^{-1}\|v\|_2^2$
and $\Tr(F^\star)/\lambda_{\min}(F^\star)=\kappa_F\,d_{\eff}(F^\star)$:
\[
\min_{t<T}\|\nabla L(\theta_t)\|_{(F^\star)^{-1}}^2
\;\le\;
C\!\left(
\frac{\Delta L}{\eta T\lambda_{\min}(F^\star)}
+\frac{\eta}{b}\,\kappa_F\,d_{\eff}(F^\star)
\right)\!\log\!\frac{1}{\delta},
\]
establishing \eqref{eq:complexity_main}.
For \eqref{eq:oracle_calls}, balance the two terms by taking
$\eta^\star=\sqrt{\Delta L\,b/\bigl(T\,\lambda_{\min}(F^\star)\,\kappa_F\,d_{\eff}\bigr)}$,
which gives the balanced RHS $\sim\sqrt{\Delta L\,\kappa_F\,d_{\eff}/(T b\,\lambda_{\min})}\,\log(1/\delta)$,
i.e., the squared-norm gap decays as $N^{-1/2}$ in $N=Tb$.  Requiring this
squared gap to be at most $\varepsilon^2$ yields the $\varepsilon^{-4}$
scaling of \eqref{eq:oracle_calls}. \Halmos
\end{proof}

\begin{remark}[Statistical vs.\ Euclidean conditioning]
\label{rem:stat_vs_euclidean}
Classical oracle-complexity results use the Hessian condition number
$\kappa_H$ and ambient dimension $d$; here they are replaced by the Fisher
counterparts $\kappa_F$ and $d_{\eff}$.
The substitution is not cosmetic: a problem may be Euclidean-stiff
($\kappa_H\gg 1$) yet statistically well-conditioned ($\kappa_F=O(1)$) when
curvature and information share the same eigenstructure; conversely, isotropic
curvature does not guarantee favorable statistical complexity.
\end{remark}

\begin{remark}[General losses / Godambe geometry]
\label{rem:complexity_godambe}
For general losses, replace $F^\star$ by $G^\star(\theta^\star)$ throughout and
interpret the dual norm accordingly. The alignment condition
\eqref{eq:complexity_noise} is then exactly the local conclusion of
Theorem~\ref{thm:godambe-alignment}.
\end{remark}

\begin{remark}[Iteration vs.\ oracle complexity]
Equation~\eqref{eq:oracle_calls} counts \emph{oracle calls} $N=Tb$. If reporting
iteration complexity, keep $b$ explicit and interpret $T=N/b$.
\end{remark}

\begin{remark}[Interpreting the complexity constants]
\label{rem:complexity_interpretation}
The product $\kappa_F\,d_{\eff}$ in \eqref{eq:oracle_calls} measures the total cost of
resolving all statistically relevant directions at the hardest
direction's scale (see Remark~\ref{rem:stable-rank} for the role of
$d_{\eff}$ as the stable rank of $F^\star$).
Equation~\eqref{eq:oracle_calls} links the sampling budget $N$ to achievable
stationarity through intrinsic problem geometry.
\end{remark}

\begin{corollary}[Batch-size scaling under a fixed sample budget]
\label{cor:batch-scaling}
Fix a total sampling budget $N$ and run $T=N/b$ iterations of constant-stepsize
SGD with batch size $b$.  Balancing the bias term
$\sim(L(\theta_0)-L^\star)/(\eta T\lambda_{\min}(F^\star))$ and the
variance term $\sim\eta\kappa_F d_{\eff}/b$ in
\eqref{eq:complexity_main} yields the balanced stepsize
$\eta^\star\propto\sqrt{b/(T\kappa_F d_{\eff})}$.
\emph{Feasibility.}  The balanced stepsize must respect the theorem's own
step-size cap $\eta\le 1/(4M)$.  Substituting $T=N/b$ gives
$\eta^\star = b\sqrt{\Delta L/(N\,\lambda_{\min}(F^\star)\,\kappa_F\,d_{\eff})}$,
so the cap $\eta^\star\le 1/(4M)$ imposes the feasibility condition
\[
b\;\le\;\sqrt{\frac{N\,\lambda_{\min}(F^\star)\,\kappa_F\,d_{\eff}}{16\,M^2\,\Delta L}}.
\]
The batch-invariance conclusion below is a statement within this
feasibility region (so $b$ may grow like $\sqrt{N}$, not linearly in $N$),
not at arbitrary $b$.  Substituting $T=N/b$ and taking $\eta=\eta^\star$
within the feasibility region, the resulting \emph{squared} Fisher-dual
stationarity gap satisfies
\[
\min_{t<T}\|\nabla L(\theta_t)\|_{(F^\star)^{-1}}^2
\;=\;
\cO\!\left(\sqrt{\frac{\kappa_F\,d_{\eff}}{N}}\right),
\]
so the norm gap decays as $N^{-1/4}$.  Within the feasibility region
the bound is independent of $b$ at leading order, which suggests---under
a fixed oracle budget---that small batches do not degrade the
worst-case guarantee while gaining more frequent iterates; the
identified geometry determines the \emph{constant} but not the
\emph{scaling} in $N$.
\end{corollary}

\section{Numerical Validation}
\label{sec:numerical}

The empirical program is organized in three layers, each answering a
distinct question about the identified covariance and its downstream
consequences:
\begin{enumerate}[label=(\roman*),leftmargin=2em]
\item \textbf{Direct identification (Experiment~1).}  Does
$\Cov(g_B(\theta^\star))\approx b^{-1}G^\star(\theta^\star)$ hold in an
actual mini-batching setup---including under model misspecification,
where curvature and noise geometry diverge?
\item \textbf{Fluctuation-scale geometric illustrations (Experiment~2).}
What does the Lyapunov structure implied by the identified covariance
\emph{look} like in a low-dimensional setting where individual
eigendirections can be visualized?
\item \textbf{Raw-scale quantitative validations (Experiments~3--6).}
Do the original-iterate-scale predictions---discrete Lyapunov plateau
at constant step, $1/N$ rate under Robbins--Monro schedules, directional
amplification, and the fluctuation-scale $(\eta/b)$-collapse (signature
of Theorem~\ref{thm:diffusion_limit})---hold quantitatively in direct
SGD on higher-dimensional problems?
\end{enumerate}
Layer~(i) is the direct check of Theorem~\ref{thm:godambe-alignment}.
Layer~(ii) is interpretive: it
visualizes the fluctuation-scale Lyapunov geometry on the normalized
time-changed OU process of Theorem~\ref{thm:diffusion_limit}, and is
\emph{not} a direct Monte Carlo verification of raw SGD.  Layer~(iii)
carries the main empirical weight: Experiments~3--6 are direct SGD
recursions on the original iterate scale, and are the quantitative
counterparts of the Lyapunov, rate, directional, and scaling-collapse
predictions derived in Sections~\ref{sec:diffusion}--\ref{sec:complexity}.
Additional robustness diagnostics (approximate exchangeability) are
deferred to Appendix~\ref{app:robustness}; reproducibility details
(seeds, shared parameters) are in the code archive.

\subsection{Claim 1: The Mini-Batch Covariance Is Identified by the Sampling Design}
\label{subsec:claim_identification}

Before probing downstream consequences, we verify the covariance
identification of Theorem~\ref{thm:godambe-alignment} directly.

\subsubsection{Experiment 1: Direct Covariance Identification (Probit DGP, Logistic Fit).}
\label{subsec:exp1_identification}

We generate $n=50{,}000$ observations from a probit data-generating
process and fit a logistic model, so that the optimization target is the
KL projection $\beta^\star$ rather than the true $\beta_0$.  Under
misspecification the score covariance and curvature no longer coincide:
$J^\star:=\Var(\nabla_\beta\ell(\beta;X,Y)|_{\beta^\star})\neq H^\star:=
\nabla^2\E[-\ell(\beta;X,Y)]|_{\beta^\star}$.  The covariance-identification
theorem nevertheless predicts
$\Cov(g_B(\beta^\star))\approx b^{-1}J^\star$, where $J^\star$ plays the
role of $G^\star(\theta^\star)$ in the general-loss notation.  We
estimate $J^\star$ from the full-sample scores and compare its diagonal
entries to those of the empirical mini-batch covariance (batch size
$b=256$, 800 replicates).  Figure~\ref{fig:partD_misspecification} shows
close agreement (relative Frobenius difference $\approx 0.05$): the
identification holds even under misspecification, and mini-batch
fluctuations are governed by the score covariance rather than by
curvature alone.

\begin{figure}
\FIGURE
{\includegraphics[width=0.68\linewidth]{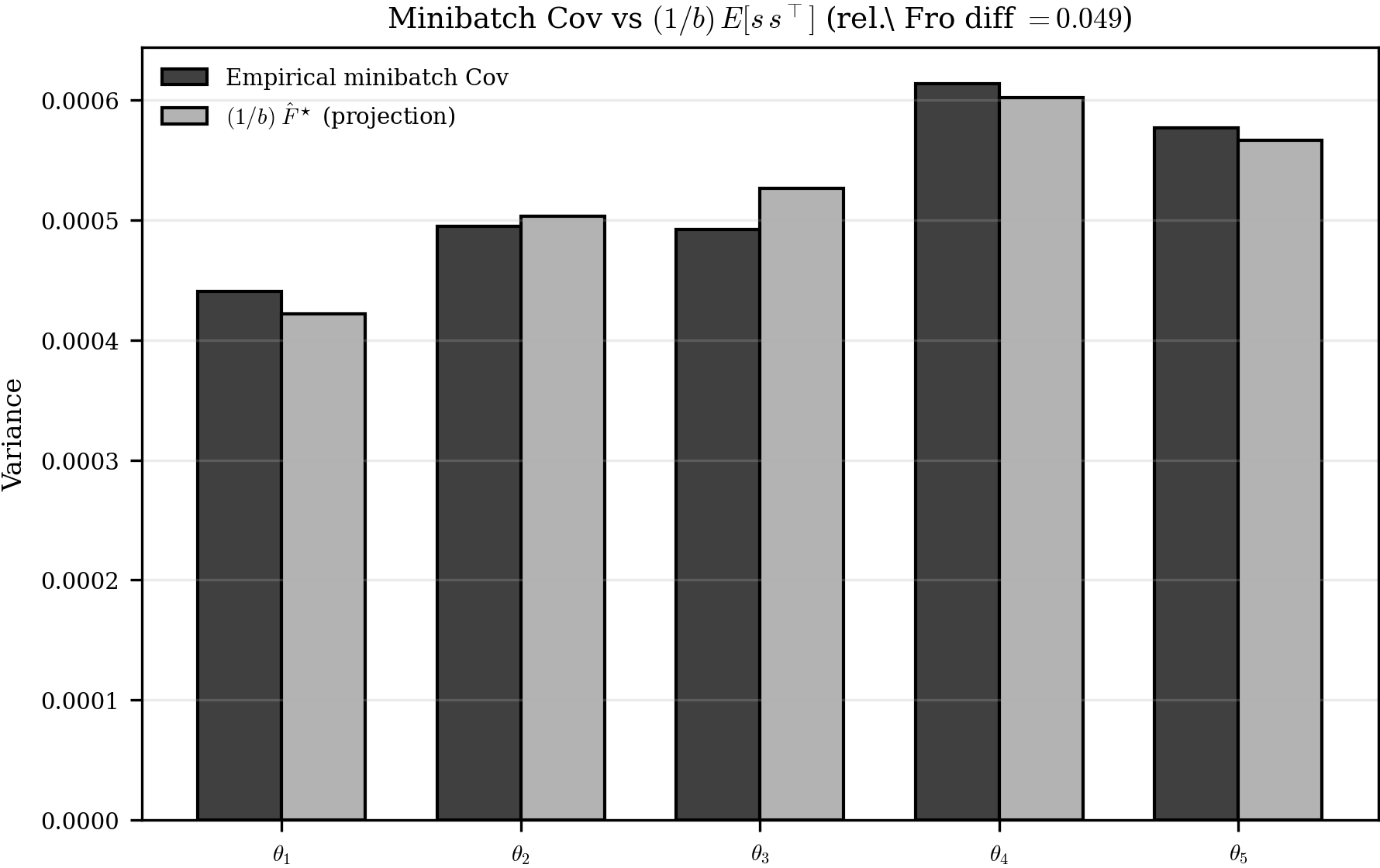}}
{Experiment~1: direct covariance identification under misspecification
(probit DGP, logistic fit).  Diagonal entries of the empirical mini-batch
covariance vs.\ $b^{-1}\hat J^\star$ (projected score covariance at
$\beta^\star$); relative Frobenius difference $\approx 0.05$.
\label{fig:partD_misspecification}}
{}
\end{figure}

\subsection{Claim 2: The Lyapunov Geometry Is Visible in Low Dimensions}
\label{subsec:claim_geom}

Once the identification is empirically in hand (Experiment~1), it is
useful to visualize what the Lyapunov structure \emph{looks like} in a
two-dimensional setting before moving to the $d=10$ quantitative tests.
In the fluctuation-scale OU surrogate of
Theorem~\ref{thm:diffusion_limit}, the stationary covariance $\Sigma_U$
solves $H^\star\Sigma_U+\Sigma_U H^{\star\top}=G^\star$, and
Experiment~2 illustrates three consequences that are each most naturally
displayed in $d=2$: the $1/b$ scaling of marginal variance, the
redistribution of variance under curvature--noise misalignment, and the
failure of trace-matched isotropic surrogates to reproduce directional
structure.
\emph{These are interpretive illustrations, not direct quantitative
validation of SGD:} the simulator integrates the time-changed OU process
of \eqref{eq:ou_batch_num} rather than running SGD iterates, and these
subplots are included because a 2D visualization of the Lyapunov
geometry aids intuition for the higher-dimensional quantitative tests
that follow.  The quantitative burden is carried by Experiments~3--6,
which run direct SGD recursions.

\subsubsection{Experiment 2: 2D Normalized OU Visualization of Lyapunov Structure.}
\label{subsec:exp2_2dlyap}

For visualization we simulate a \emph{pedagogical normalized OU process}
chosen so that marginal variance scales as $1/b$ and is independent of
$\eta$ at leading order:
\begin{equation}
d\theta_t \;=\; -\eta\,H\,\theta_t\,dt \;+\; \sqrt{\frac{2\eta}{b}}\;F^{1/2}\,dW_t,
\label{eq:ou_batch_num}
\end{equation}
with $H=\mathrm{diag}(1,0.1)$ and varying noise shape $F$.  This
normalization is chosen for visualization only; it is \emph{not} the
literal fluctuation SDE of Theorem~\ref{thm:diffusion_limit}, which
scales out the $b$-dependence into the $\sqrt{b/\eta}$ rescaled process
$U^\eta$.  At stationarity, the continuous-time Lyapunov equation for
\eqref{eq:ou_batch_num} reads $H\Sigma+\Sigma H=(2/b)F$, and the
resulting marginal variances visualize the $1/b$ scaling, curvature--noise
misalignment effects, and rotation diagnostics of the Lyapunov geometry
in two dimensions.

\paragraph{Batch-size sweep ($1/b$ scaling).}
With $F=I_2$ and a wide sweep of $b$,
Figure~\ref{fig:partA_lyapunov_batch_scaling} overlays simulated
stationary variances with the fluctuation-scale Lyapunov prediction:
clean $1/b$ scaling holds across the sweep, with the flat coordinate
mixing more slowly (weaker curvature) but reaching the same equilibrium
law.

\begin{figure}
\FIGURE
{\includegraphics[width=\linewidth]{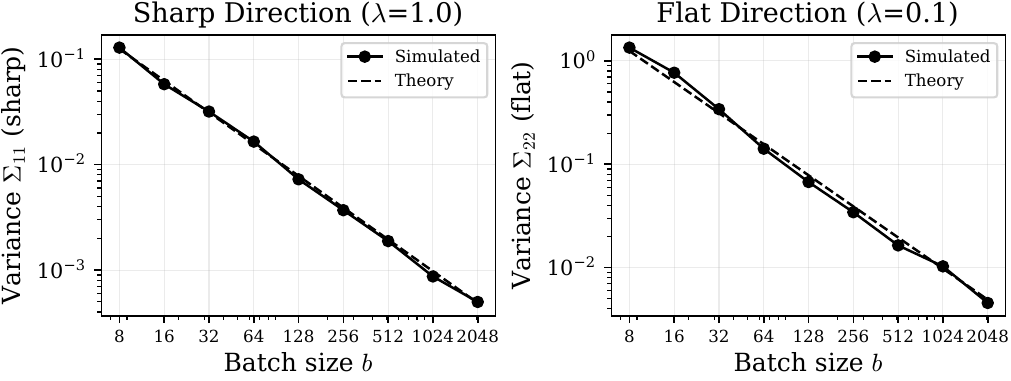}}
{Experiment~2, 2D batch-size sweep: stationary variances vs.\ $b$ (log
scales); solid = simulation, dashed = fluctuation-scale Lyapunov prediction.
\label{fig:partA_lyapunov_batch_scaling}}
{$d=2$, $H=\mathrm{diag}(1,0.1)$, $F=I_2$, 500 replicates, 200k steps.}
\end{figure}

\paragraph{Angle sweep (curvature--noise misalignment).}
Rotating an anisotropic noise matrix
$F(\varphi)=R(\varphi)\mathrm{diag}(1.5,0.5)R(\varphi)^\top$ relative to
the curvature basis, Figure~\ref{fig:partB_anisotropic_alignment} shows
a predictable transfer of stationary variance across coordinates as
$\varphi$ varies: alignment governs which directions receive noise
energy at equilibrium, even when the scalar temperature $\eta/b$ is held
fixed.  The discrete Lyapunov benchmark tracks the simulation closely.

\begin{figure}
\FIGURE
{\includegraphics[width=\linewidth]{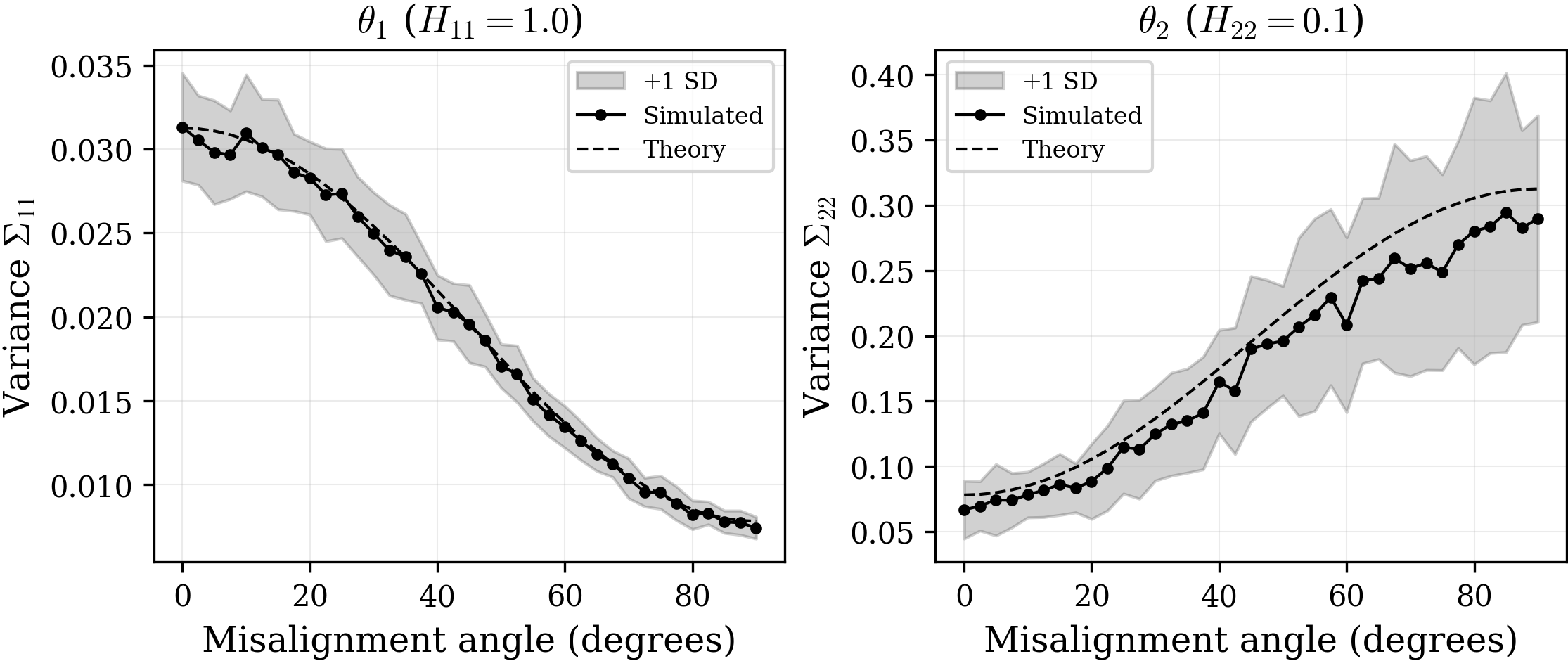}}
{Experiment~2, 2D angle sweep: stationary variances vs.\ misalignment
angle $\varphi$.  Bands show replicate variability; dashed curves are
the discrete Lyapunov prediction.
\label{fig:partB_anisotropic_alignment}}
{$d=2$, $H=\mathrm{diag}(1,0.1)$, 500 replicates per angle.}
\end{figure}

\paragraph{Rotation diagnostic (geometry vs.\ scalar temperature).}
We compare a geometric OU with rotating diffusion
$F_\star(\varphi)$ against a trace-matched isotropic surrogate
$\sigma^2 I$ with $\sigma^2=\Tr F_\star/d$.
Figures~\ref{fig:partC_ellipses}--\ref{fig:partC_crosscov} show that the
geometric stationary covariance $\Sigma_F(\varphi)$ tilts with
$\varphi$ (nonzero cross-covariance, redistributed marginals), while
the isotropic surrogate remains axis-aligned at all angles.  This is
the signature the trace-matched surrogate \emph{cannot} reproduce, and
is the low-dimensional preview of the directional amplification
measured in $d=10$ by Experiment~5.

\begin{figure}
\FIGURE
{\includegraphics[width=\linewidth]{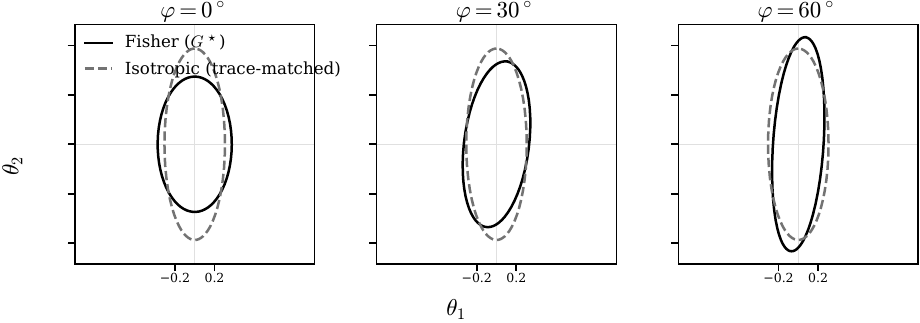}}
{Experiment~2, 2D rotation diagnostic: stationary covariance ellipses at
three rotation angles.  The geometric equilibrium tilts with $\varphi$;
the isotropic trace-matched surrogate remains axis-aligned.
\label{fig:partC_ellipses}}
{$d=2$, $b=64$; ellipses are 95\% level sets of the Lyapunov solution.}
\end{figure}

\begin{figure}
\FIGURE
{\includegraphics[width=\linewidth]{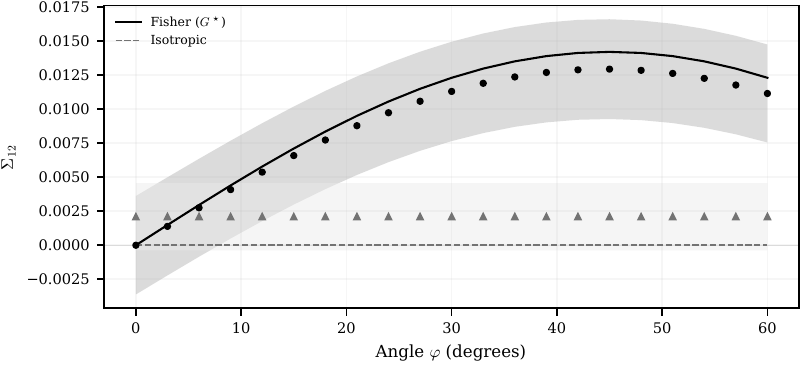}}
{Experiment~2, 2D rotation diagnostic: cross-covariance $\Sigma_{12}$
vs.\ rotation angle.  Isotropic trace-matched noise fails to reproduce
off-diagonal covariance; bands show $\pm 1$ s.d.\ across seeds.
\label{fig:partC_crosscov}}
{$d=2$, 6 seeds, $T=2{,}500$, burn-in $= 800$.}
\end{figure}

\subsection{Claim 3: The Lyapunov Plateau Holds Quantitatively in $d=10$}
\label{subsec:claim_plateau}

The 2D experiments display the qualitative geometry; we now move to
$d=10$ for the first high-dimensional quantitative checkpoint.  This is
where the \emph{quantitative} weight of the Lyapunov prediction is
tested: Experiment~3 runs direct SGD on a $d=10$ problem and asks
whether the empirically observed steady-state Fisher risk matches the
level predicted by combining the identified covariance $G^\star$ with
the discrete Lyapunov equation (Proposition~\ref{prop:discrete_lyap}).
If the identification from Theorem~\ref{thm:godambe-alignment} is
quantitatively correct, both the anisotropic and isotropic empirical
risks should converge to the Lyapunov-predicted plateau
$\Tr(F^\star\Sigma_\eta^{\mathrm{lin}})$ as $T$ grows, with convergence
ratios $\widehat R_T/R_\infty\to 1$.

\subsubsection{Experiment 3: OU/Lyapunov Plateau in $d=10$ (Anisotropic vs.\ Isotropic).}
\label{subsec:exp3_ou_plateau}

The curvature is $H=\mathrm{diag}(\lambda_1,\ldots,\lambda_{10})$ with
eigenvalues log-spaced between~1 and~100, and we set $F^\star=H^\star$
throughout the synthetic $d=10$ benchmark (the quadratic-Gaussian
surrogate in which curvature and Fisher information coincide, so the
Fisher-metric risk $\E[\|\theta_T\|_{F^\star}^2]$ directly measures
Hessian-weighted error).  The noise shape $G^\star$ is a random SPD
matrix with eigenvalues in $[0.5,5]$, and we compare an isotropic
control $G_{\mathrm{iso}}=\sigma^2 I$ with $\sigma^2=\Tr(G^\star)/d$
(trace-matched).  Under constant step size $\eta=0.5/\lambda_{\max}(H)$
and batch size $b=64$, we run $R=2{,}000$ replicates and record the
Fisher-metric risk $\E[\|\theta_T\|_{F^\star}^2]$ over increasing
horizons $T\in\{300,800,2000,5000,10000\}$.

Table~\ref{tab:ou_plateau} confirms the Lyapunov prediction: both the
anisotropic and isotropic empirical risks converge to the predicted
stationary level as $T$ grows, with convergence ratios
$\widehat R_T/R_\infty$ within $1\%$ of unity across all horizons.  The
two plateaux nearly coincide because trace-matching is exact in this
diagonal-$H$ setting; the distinction between anisotropic and isotropic
geometry is therefore invisible at the level of scalar risk, and only
emerges in higher-order structure (Experiment~5).

\begin{table}[t]
\centering\small
\begin{tabular}{@{}rcccccc@{}}
\toprule
 & \multicolumn{2}{c}{Empirical Risk $\widehat{R}_T$} &
   \multicolumn{2}{c}{Lyapunov $R_\infty$} &
   \multicolumn{2}{c}{$\widehat{R}_T / R_\infty$} \\
\cmidrule(lr){2-3}\cmidrule(lr){4-5}\cmidrule(lr){6-7}
$T$ & Aniso & Iso & Aniso & Iso & Aniso & Iso \\
\midrule
    300 & $8.55\!\times\!10^{-4}$ & $8.60\!\times\!10^{-4}$ & $8.63\!\times\!10^{-4}$ & $8.58\!\times\!10^{-4}$ & 0.991 & 1.002 \\
    800 & $8.64\!\times\!10^{-4}$ & $8.49\!\times\!10^{-4}$ & & & 1.002 & 0.990 \\
  2\,000 & $8.58\!\times\!10^{-4}$ & $8.50\!\times\!10^{-4}$ & & & 0.995 & 0.992 \\
  5\,000 & $8.72\!\times\!10^{-4}$ & $8.58\!\times\!10^{-4}$ & & & 1.010 & 1.001 \\
 10\,000 & $8.63\!\times\!10^{-4}$ & $8.47\!\times\!10^{-4}$ & & & 1.000 & 0.987 \\
\bottomrule
\end{tabular}
\caption{Experiment~3 (OU/Lyapunov plateau, $d=10$): empirical
Fisher-metric risk vs.\ discrete Lyapunov prediction
$R_\infty=\Tr(F^\star\Sigma_\eta^{\mathrm{lin}})$.  Convergence ratios
$\widehat{R}_T/R_\infty\to 1$ confirm quantitative accuracy.
\label{tab:ou_plateau}}
{\footnotesize $d=10$, $H=\mathrm{diag}(1,\ldots,100)$, random $G^\star$
with eigenvalues in $[0.5,5]$, $\eta=0.5/\lambda_{\max}$, $b=64$,
$R=2{,}000$ replicates.}
\end{table}

\subsection{Claim 4: Decaying Step Sizes Recover a $1/N$ Rate with a Geometry-Determined Constant}
\label{subsec:claim_rate}

Under a Robbins--Monro step-size schedule the Fisher-metric risk decays
as $C/N$ where $N=Tb$ is the total oracle budget and the constant $C$
is determined by the intrinsic geometry $(G^\star,H^\star,F^\star)$.
Specifically, the SA central-limit theorem gives an asymptotic
covariance $V$ satisfying $({\eta_0 H^\star-\tfrac12 I})V+V({\eta_0
H^\star-\tfrac12 I})=\frac{\eta_0^2}{b}G^\star$, and $N\cdot\mathrm{Risk}
\to b\,\Tr(F^\star V)$.  This is consistent with the upper bound
(Theorem~\ref{thm:upper-bound}) in rate order and with the i.i.d.\
parametric Fisher lower-bound benchmark
(Proposition~\ref{thm:lower-bound}).

\subsubsection{Experiment 4: $1/N$ Rate and Fisher Constant (Decaying Step Size).}
\label{subsec:exp4_rate}

We switch to a decaying step size $\eta_t=\eta_0/(t+t_0)$ with
$\eta_0=1.2$, $t_0=50$, and the same $(H,G^\star,G_{\mathrm{iso}})$ from
Experiment~3, with oracle budget $N=Tb$ and $b=64$.
Figure~\ref{fig:risk_rate} (left) plots the Fisher-metric risk on
log--log axes: both anisotropic and isotropic risks decay as $1/N$
(fitted slopes $\approx -0.98$), consistent with
Theorem~\ref{thm:upper-bound} in rate order.
The right panel shows the scaled risk $N\cdot\widehat{\mathrm{Risk}}$,
which stabilizes near $\approx 14.2$ (anisotropic) and $\approx 14.0$
(isotropic)---both close to the SA--CLT prediction $b\,\Tr(F^\star V)$
where $V$ solves the Lyapunov equation above.

\begin{figure}
\FIGURE
{\includegraphics[width=\linewidth]{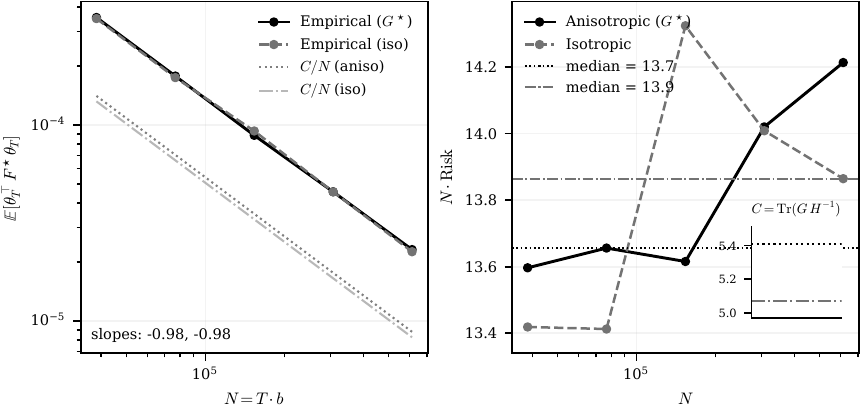}}
{Experiment~4 (decaying step size, $d=10$).
\textbf{Left:} Fisher-metric risk vs.\ $N$ (log--log); slopes
$\approx -1$.  \textbf{Right:} Scaled risk $N\cdot\widehat{\mathrm{Risk}}$
stabilizes; inset shows $C=\Tr(G^\star H^{-1})$.
\label{fig:risk_rate}}
{$\eta_t=1.2/(t+50)$, $b=64$, 800 replicates.}
\end{figure}

\emph{Interpretation: two related-but-distinct constants.}  The
constant-step Lyapunov quantity $\Tr(G^\star H^{\star-1})$ (inset,
$\approx 5.41$) and the decaying-step asymptotic constant
$b\,\Tr(F^\star V)$ (plateau, $\approx 14.0$--$14.2$) are related but not
identical: the former is the stationary
$\Tr(F^\star\Sigma_\eta^{\mathrm{lin}})$ of the constant-step discrete
Lyapunov recursion at leading order in $\eta$, while the latter is the
$N\to\infty$ asymptotic constant of the scaled risk under a
Robbins--Monro schedule, which solves a \emph{different} Lyapunov
equation shifted by $-\tfrac12 I$.  The shift reflects the variance
deflation introduced by the decaying schedule.  Both constants are
determined by the identified geometry $(G^\star,H^\star,F^\star)$, but
their numerical values differ, and the ratio between them depends on the
schedule (here $\eta_0=1.2$, $t_0=50$).  The empirical scaled risk
$N\cdot\widehat{\mathrm{Risk}}$ lies above the i.i.d.\ parametric
Fisher lower-bound benchmark of Proposition~\ref{thm:lower-bound},
consistent with rate-order matching; we do not claim a specific value
for the theorem constant.

\subsection{Claim 5: Directional Amplification Discriminates Geometry from Scalar Temperature}
\label{subsec:claim_directional}

Experiments~3 and~4 compared anisotropic and isotropic noise at the
level of trace-scale summaries, where trace-matching makes the two
surrogates nearly indistinguishable.  This is precisely where
scalar-temperature reasoning fails.  In the synthetic benchmark
$F^\star=H^\star$, so the plotted Fisher-metric risk is
$\Tr(F^\star\Sigma)=\Tr(H^\star\Sigma)$.  Multiplying the
fluctuation-scale Lyapunov equation $H^\star\Sigma_U+\Sigma_U H^{\star\top}=G^\star$
on the left by the identity and taking traces gives
$2\,\Tr(H^\star\Sigma_U)=\Tr(G^\star)$, which depends on $G^\star$ only
through its trace; a trace-matched isotropic surrogate
($G_{\mathrm{iso}}$ with $\Tr G_{\mathrm{iso}}=\Tr G^\star$) therefore
reproduces the scalar Fisher-risk summary by construction, even though
its \emph{directional} distribution of residual risk is entirely
different.  The discriminating signature is
\emph{directional}: anisotropic noise concentrates stationary risk
along the top eigendirections of $G^\star$, while the isotropic
surrogate spreads residual error uniformly.  Experiment~5 measures
this direction-by-direction.

\subsubsection{Experiment 5: Directional Variance Amplification.}
\label{subsec:exp5_directional}

For $k=1,2,3$ we project $\theta_T$ onto the top eigenvectors $v_k$ of
$G^\star$ and compute the directional second moment
$\E[\langle v_k,\theta_T\rangle^2]$ under the same decaying-step setup
as Experiment~4.  Figure~\ref{fig:directional} (left) shows that both
models decay as $1/N$, but the anisotropic curve is systematically
above the isotropic one along the top eigendirections.  The right
panel plots the ratio $\E_{\mathrm{aniso}}/\E_{\mathrm{iso}}$: at the
largest budget the empirical ratios are $(1.76, 1.33, 1.35)$ for
$k=1,2,3$, close to the linearized Lyapunov prediction $(1.77, 1.40, 1.33)$
obtained from $H^\star\Sigma_U+\Sigma_U H^{\star\top}=G^\star$.  The
ratios are strictly greater than one in every direction and largest
along the top eigenvector of $G^\star$.  Trace-matched isotropic noise matches the scalar-risk summaries of
Experiments~3--4 but cannot reproduce this direction-by-direction risk
allocation; the identified geometry determines which parameter
directions carry the most residual risk.

\begin{figure}
\FIGURE
{\includegraphics[width=\linewidth]{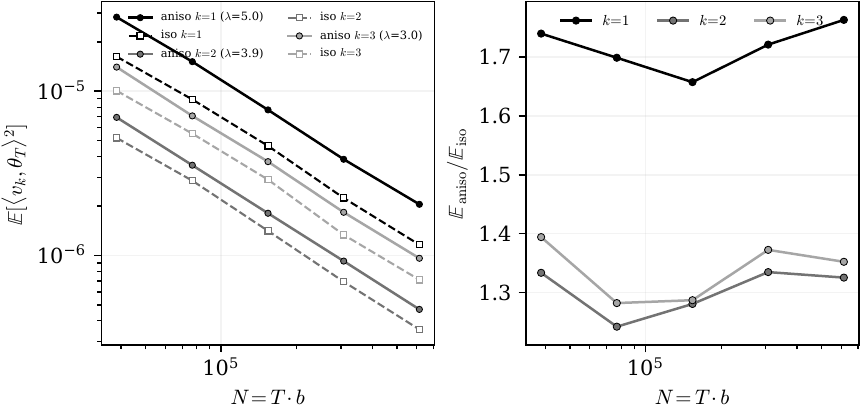}}
{Experiment~5 (directional variance amplification, $d=10$).
\textbf{Left:} Directional second moments along top eigenvectors of
$G^\star$.
\textbf{Right:} Ratio of anisotropic to isotropic directional variances;
ratios $>1$ mirror the eigenvalue ordering of $G^\star$.
Trace-matched isotropic noise preserves $\mathrm{Tr}(G^\star)$ but
spreads residual risk uniformly; the identified anisotropic noise
concentrates it along the top noise eigendirections. \label{fig:directional}}
{Same setup as Experiment~4; top $k=1,2,3$ eigenvectors of $G^\star$.}
\end{figure}

\subsection{Claim 6: Fluctuation-Scale $(\eta/b)$-Collapse on the Raw Iterate Path}
\label{subsec:claim_collapse}

Theorem~\ref{thm:diffusion_limit} and Corollary~\ref{cor:lyapunov} predict
that, in the small-stepsize local regime,
$\mathbb E[\|\theta_T-\theta^\star\|_{F^\star}^2]=(\eta/b)\,\Tr(F^\star\Sigma_U)+o(\eta/b)$
with $\Sigma_U$ solving $H^\star\Sigma_U+\Sigma_U H^{\star\top}=G^\star(\theta^\star)$.
Equivalently, the scaled Fisher risk $(b/\eta)\,\widehat{\mathrm{Risk}}$
should collapse to a single $\eta,b$-free curve at leading order.  This
is the direct empirical signature of the fluctuation FCLT
(Theorem~\ref{thm:diffusion_limit}) on the original iterate scale.

\subsubsection{Experiment 6: Raw-Scale Collapse Diagnostic.}
\label{subsec:exp6_collapse}

We sweep eight $(\eta,b)$ configurations with $\eta\in\{0.001,0.002,0.005\}$,
$b\in\{32,64,128\}$ (so $\eta\lambda_{\max}(H^\star)\in\{0.1,0.2,0.5\}$;
$\eta/b\in[8\times10^{-6},1.6\times10^{-4}]$), and run direct SGD on the
same $(H,F^\star,G^\star)$ as Experiments~3--4 for
$T\in\{2000,5000,10000,20000\}$ with $R=600$ replicates per cell.
Figure~\ref{fig:partG_collapse} plots
$(b/\eta)\,\widehat{\mathbb E}[\|\theta_T\|_{F^\star}^2]$ against $T$ for
each configuration.  All curves collapse onto a common plateau near
$\Tr(F^\star\Sigma_U)\approx 10.22$ (black reference line), the
fluctuation-scale Lyapunov prediction.  The collapse across two orders
of magnitude in $\eta/b$ is a direct verification of the
$(\eta/b)$-scaling law on the raw iterate scale, independently of any
continuous-time normalization convention.

\begin{figure}
\FIGURE
{\includegraphics[width=\linewidth]{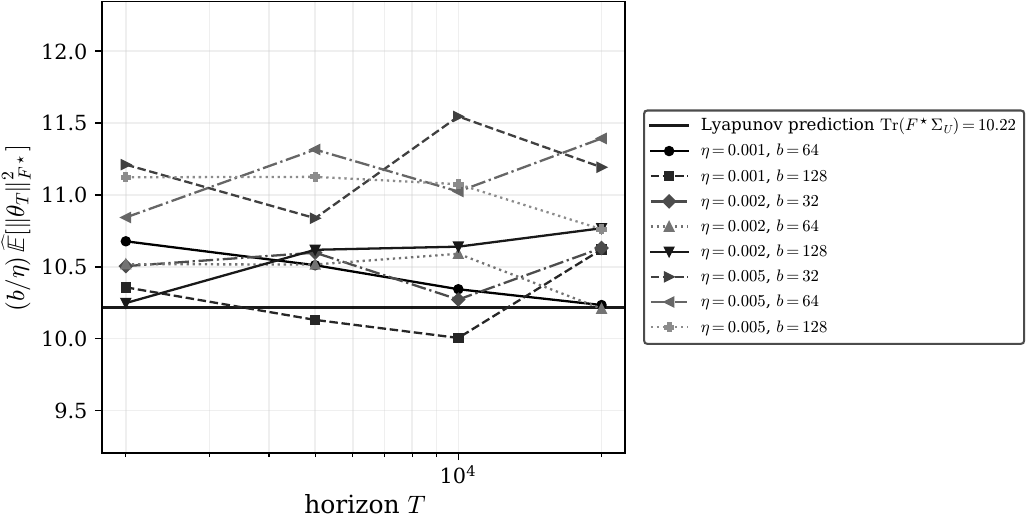}}
{Experiment~6 (fluctuation-scale collapse diagnostic; Theorem~\ref{thm:diffusion_limit}):
scaled Fisher risk $(b/\eta)\,\widehat{\mathbb E}[\|\theta_T\|_{F^\star}^2]$
stabilizes at the fluctuation-scale Lyapunov prediction
$\Tr(F^\star\Sigma_U)$ (black line) across eight $(\eta,b)$ configurations.
\label{fig:partG_collapse}}
{$d=10$, direct SGD, $R=600$ replicates per cell;
$\eta\in\{0.001,0.002,0.005\}$, $b\in\{32,64,128\}$.}
\end{figure}

\medskip
Additional robustness experiments under approximate exchangeability are
reported in Appendix~\ref{app:robustness}.

\section{Discussion}
\label{sec:discussion}

\subsection{Summary of Contributions}

The paper's central contribution is operational: we identify the
mini-batch gradient covariance from the sampling design itself as
$b^{-1}G^\star(\theta)$, converting the diffusion analysis of
constant-step SGD from a modeling exercise into a statement about the
replication budget.  In likelihood models the identified matrix reduces
to projected Fisher information; in general losses it pairs with
$H^\star$ through the classical sandwich/Godambe geometry of
M-estimation.  At fixed batch size, the raw iterate path has a
deterministic fluid limit (the gradient-flow ODE); the
$\sqrt{b/\eta}$-scaled fluctuations converge to a linear diffusion with
coefficient $G^\star$; near a nondegenerate optimum this is
Ornstein--Uhlenbeck, and the corresponding Lyapunov covariance scaled
by $\eta/b$ matches the stationary covariance of the linearized
discrete recursion at leading order.  In the local Fisher case, an
upper bound of order $1/N$ is rate-order matched by an i.i.d.\ van~Trees
lower bound, with the intrinsic dimension $d_{\mathrm{eff}}$ and
statistical condition number $\kappa_F$ appearing in the upper bound
only.

\subsection{Operational Implications for Sampling-Budget Design}

The effective temperature $\tau=\eta/b$ is the single scalar that
converts replication effort into diffusion amplitude, and the
identification shows it does so in a \emph{directional} way fixed by
$G^\star$ rather than uniformly.  For an OR practitioner, three
consequences follow at the sampling-budget level.  First, at a fixed
oracle budget $N=Tb$, the batch size $b$ trades iterate count against
per-step noise scale: increasing $b$ cools the error floor predicted by
the linearized discrete Lyapunov benchmark
(Proposition~\ref{prop:discrete_lyap}) in magnitude, but leaves the noise
ellipsoid's \emph{shape} unchanged.
Second, classical OR variance-reduction interventions (common variates,
control variates, stratification) should be evaluated by their Loewner
effect on $G^\star$, and thus by the induced reduction in
$\Tr(Q\Sigma)$ for the operational metric $Q$, rather than by scalar
component-variance alone---this is the natural bridge from classical
variance-reduction practice to mini-batch SGD.  Third, the directional
amplification in Experiment~5 identifies exactly which parameter
components absorb residual risk, and thus where additional sampling
effort has the largest marginal return.

\begin{remark}[Variance reduction as Fisher-metric risk reduction]
\label{rem:vr-fisher}
A VR technique that reduces $G^\star$ to $\widetilde G\preceq G^\star$
in the Loewner order yields $\widetilde\Sigma_\infty\preceq\Sigma_\infty$
via the Lyapunov balance, and hence lower Fisher-metric risk.
\end{remark}

Plain SGD does not precondition the drift; natural-gradient methods do
so by explicitly reweighting $\nabla L$.  We do not claim plain SGD is
``natural-gradient-like'' in the drift sense; rather, its linearized
local equilibrium benchmark (Proposition~\ref{prop:discrete_lyap})
reflects the noise covariance $G^\star$ through the Lyapunov balance.
The identification also suggests adaptive batching---increasing $b$
(cooling) when marginal noise reduction is valuable, decreasing it when
exploration matters.

\subsection{Practical Notes and Extensions}
\label{subsec:practical_notes}

The preceding theory identifies a small set of geometric
primitives---$H^\star$, $G^\star$ (or $F^\star$), and
$\tau=\eta/b$---that govern both transient rates and stationary risk.

Our analysis freezes $F(\theta)$ locally at $F^\star$.
Validity can be assessed by monitoring
$\delta_F(t):=\|F(\theta_t)-F^\star\|_{\op}/\|F^\star\|_{\op}$; in our
experiments the Lyapunov predictions remain accurate from random starts,
indicating early entry into the basin of Fisher geometry.

\subsection{Estimating $G^\star$ in Practice}
\label{subsec:estimating_Gstar}

Where $G^\star$ is not analytically available, the empirical covariance of
per-sample gradients $\hat G(\theta_t) = \frac{1}{m-1}\sum_i
(\psi_i-\bar\psi)(\psi_i-\bar\psi)^\top$ is the natural plug-in and
converges at the standard $O(1/\sqrt{m})$ rate in operator norm; for large
$d$, a diagonal or rank-$k$ streaming approximation suffices for the
quantities ($d_{\mathrm{eff}}$, dominant eigendirections, scale anisotropy)
that enter the bounds.  Given $\hat G$, one can diagnose whether isotropic
temperature-matching is likely to fail (by checking whether $\hat G$ is far
from a scalar multiple of $I$) and compare candidate batch sizes through
the predicted stationary covariance $\Sigma_\eta$ of
Proposition~\ref{prop:discrete_lyap}.  A natural design heuristic
(not proved here) is to regulate $\tau_t\propto 1/\lambda_{\max}(H(\theta_t))$
to avoid over-heating in stiff directions; when $d_{\eff}\ll d$, the
linearized local equilibrium of Proposition~\ref{prop:discrete_lyap}
concentrates on the identifiable subspace captured by $\Pi_\theta$,
which is where the identified geometry is nondegenerate.

\subsection{Limitations and Future Work}

Our sharpest statements use local linearization near a nondegenerate critical
point; extending to nonconvex regimes requires controlling state-dependent
curvature and relating intrinsic geometry to metastability.
When $d\gg n$, the identified covariance may be low-rank, and a satisfactory
theory should work on identifiable manifolds with appropriate
pseudoinverses.
Finally, preconditioners (Gauss--Newton, KFAC) reshape both drift and noise
geometry; a joint control theory co-designing contraction and diffusion---potentially under communication constraints---would connect information
geometry to OR-style variance allocation at scale.

The empirical design intentionally progresses from interpretable
geometric diagnostics (direct covariance identification, $2$D Lyapunov
visualization) to quantitative raw-scale validation under direct SGD
recursions in $d=10$, culminating in the fluctuation-scale
$(\eta/b)$-collapse (Experiment~6; signature of
Theorem~\ref{thm:diffusion_limit}) on the original iterate path.

\paragraph{Code availability.}
All experiments in Section~\ref{sec:numerical} and Appendix~\ref{app:robustness}
are reproducible from the accompanying code archive (\texttt{Python\_Codes/}),
with master script \texttt{run\_all.py}, explicit random seeds, and
wall-clock under 5 minutes on a single CPU.

\appendix

\section{Technical Lemmas}
\label{app:lemmas}

This appendix collects the auxiliary lemmas invoked by the main-text
proofs.  The finite-population covariance identity
(Lemma~\ref{lem:finite_pop_cov_exact}) and score identities
(Lemma~\ref{lem:score_identities_setup}) are stated and proved in the main
text; we only restate here the auxiliary results (Lyapunov representation,
Robbins--Monro product bound) that are invoked in later proofs but are not
part of the main-text spine.

\subsection{Lyapunov equation solution}

\begin{lemma}[Lyapunov equation solution]
\label{lem:lyap_rep_app}
Let $A \in \R^{d \times d}$ with all eigenvalues having strictly positive real part.
Let $Q \succeq 0$. Then the continuous Lyapunov equation
$A\Sigma + \Sigma A^\top = Q$
has the unique solution
\[
\Sigma = \int_0^\infty e^{-As} Q e^{-A^\top s} ds.
\]
\end{lemma}

\begin{proof}{Proof.}
\emph{Convergence:} Since all eigenvalues of $A$ have positive real part,
$\|e^{-As}\| \le C e^{-\lambda s}$ for some $C, \lambda > 0$, ensuring convergence.

\emph{Verification:}
$A\Sigma + \Sigma A^\top = \int_0^\infty \frac{d}{ds}(-e^{-As} Q e^{-A^\top s}) ds
= [-e^{-As} Q e^{-A^\top s}]_0^\infty = Q$.

\emph{Uniqueness:} If $\Sigma_1, \Sigma_2$ both solve the equation, then $\Delta := \Sigma_1 - \Sigma_2$
satisfies $A\Delta + \Delta A^\top = 0$, whose only solution is $\Delta = 0$ when $A$ has
eigenvalues with positive real part. \Halmos
\end{proof}
\subsection{Product bound for Robbins--Monro weights}

\begin{lemma}[Product bound for Robbins-Monro weights]
\label{lem:product_bound_app}
Let $a > 1$ and $c_0 > 0$. Define
\[
\alpha_k := 1 - \frac{a}{k} + \frac{c_0}{k^2}, \qquad \Pi_{s,T} := \prod_{k=s}^{T-1} \alpha_k.
\]
Then for all $T$ sufficiently large:
\begin{enumerate}
    \item[(i)] $\Pi_{1,T} \le C T^{-a}$ for a constant $C = C(a, c_0)$.
    \item[(ii)] $\Pi_{t+1,T} \le C (t/T)^a$ for $1 \le t < T$.
\end{enumerate}
\end{lemma}

\begin{proof}{Proof.}
Taking logarithms: $\log \alpha_k = \log(1 - a/k + c_0/k^2) = -a/k + O(1/k^2)$ for large $k$.
Thus $\log \Pi_{s,T} = -a \sum_{k=s}^{T-1} k^{-1} + O(1) = -a(\log T - \log s) + O(1)$.
Exponentiating: $\Pi_{s,T} = O((s/T)^a)$.
For $s = 1$: $\Pi_{1,T} = O(T^{-a})$.
For $s = t+1$: $\Pi_{t+1,T} = O((t/T)^a)$. \Halmos
\end{proof}

\section{Robustness: A Stylized Variance-Inflation Perturbation}
\label{app:robustness}

The following experiment probes one stylized perturbation of the
fresh-sampling assumption of Theorem~\ref{thm:godambe-alignment}: a
scalar variance-inflation factor $(1+\varepsilon)$ on the mini-batch
covariance, modeling mild departures from ideal exchangeability (weak
dependence, partial sample reuse).  This is a specific perturbation
diagnostic, not a general robustness theorem.  A direct check of the
identification itself under model misspecification is in Experiment~1
of Section~\ref{subsec:exp1_identification}.

To model mild departures from ideal exchangeability (e.g., weak dependence,
contamination, or partial reuse of samples), we inflate the intrinsic covariance
by a factor $(1+\varepsilon)$ in the linearized surrogate:
$\Cov(\xi_t)=(1/b)(1+\varepsilon)\,\Gstar(\theta^\star)$, $\varepsilon\in[0,1]$.
This captures ``extra correlation'' as increased effective diffusion amplitude.
Linearity of the Lyapunov operator then implies
$\Sigma(\varepsilon)=(1+\varepsilon)\Sigma(0)$,
so every stationary marginal variance inflates linearly in $\varepsilon$.
Figure~\ref{fig:partE_eps_exchangeable_scaling} confirms this prediction:
the entire stationary covariance scales as though the system had a
higher effective temperature
$\tau_{\mathrm{eff}}=(1+\varepsilon)\eta/b$.
To first order, small exchangeability violations act like ``heating''
the Fisher--Lyapunov equilibrium.

\begin{figure}[H]
\FIGURE
{\includegraphics[width=0.42\linewidth]{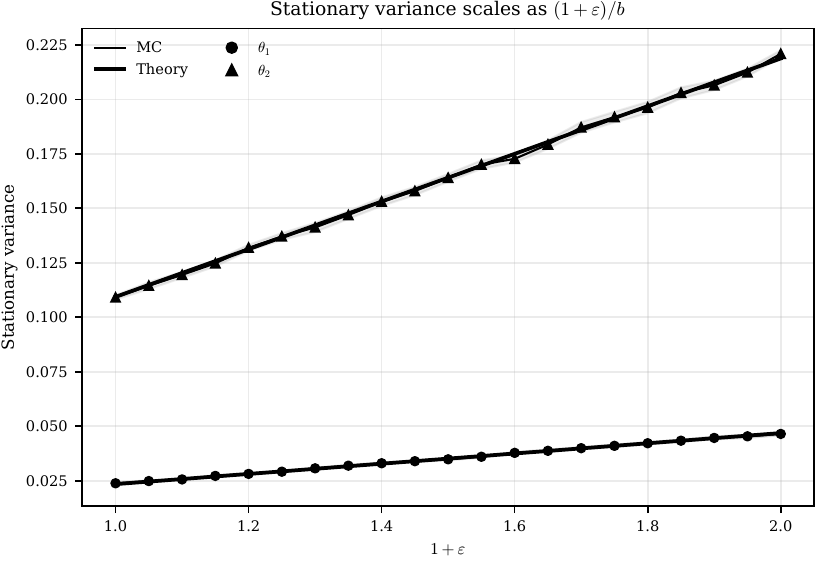}}
{Experiment~B: approximate exchangeability. Stationary variances vs.\
$\varepsilon$ (simulation bands vs.\ linear Lyapunov prediction).
\label{fig:partE_eps_exchangeable_scaling}}
{}
\end{figure}

\FloatBarrier

\bibliographystyle{plainnat}
\bibliography{references}

\end{document}